\documentclass[10pt,twocolumn,letterpaper]{article}

\usepackage{cvpr}              
\usepackage{graphicx}
\usepackage{tikz}
\usepackage{multirow}
\usepackage[accsupp]{axessibility} 
\usepackage{comment}
\usepackage{amsmath,amssymb,amsthm,mathtools} 
\usepackage{booktabs}
\usepackage{bm}
\usepackage{color}
\usepackage[abs]{overpic}
\usepackage{algorithm,algpseudocode}
\algrenewcommand\algorithmicrequire{\textbf{Input:}}
\algrenewcommand\algorithmicensure{\textbf{Output:}}

\definecolor{cvprblue}{rgb}{0.21,0.49,0.74}
\usepackage[pagebackref,breaklinks,colorlinks,citecolor=cvprblue]{hyperref}
\usepackage[capitalize]{cleveref}
\renewcommand{\vec}[1]{\boldsymbol{#1}}

\newcommand{\pose}[0]{\vec{\theta}}
\newcommand{\shape}[0]{\vec{\beta}}

\newcommand{\smpl}[0]{M}

\newcommand{\loss}[0]{\mathcal{L}}
\newcommand{\sota}{SotA}

\newcommand{\figref}[1]{Fig.~\ref{#1}}
\newcommand{\secref}[1]{Sec.~\ref{#1}}
\newcommand{\eqnref}[1]{Eq.~\eqref{#1}}
\newcommand{\tabref}[1]{Tab.~\ref{#1}}

\newcommand{\x}{\mathbf{x}}

\newcommand{\rb}{\mathbf{r}}

\newcommand{\Ab}{\mathbf{A}}

\newcommand{\G}{\mathbf{G}}
\newcommand{\Gx}{\mathbf{G}_{\mathbf{x}}}
\newcommand{\Id}{\mathbf{I}}

\newcommand{\Pm}{\mathbf{P}}

\newcommand{\z}{\mathbf{z}}

\newcommand{\q}{\mathbf{q}}

\newcommand{\eb}{\mathbf{e}}

\newcommand{\ub}{\mathbf{u}}
\newcommand{\vb}{\mathbf{v}}

\renewcommand{\v}{\mathbf{v}}
\newcommand{\J}{\mathbf{J}}

\newcommand{\bp}{\mathbf{p}}
 
\newcommand{\Exp}{\mathrm{{Exp}}}
\newcommand{\Log}{\mathrm{{Log}}}
 
\newcommand{\Quats}{\mathbb{H}_1}
\newcommand{\TqM}{\T_{\q}\Quats}
\newcommand{\TxS}{\T_{\x}\mathcal{S}}

\newcommand{\TposeQuats}{\T_{\pose}\Quats^K}

\newcommand{\y}{\mathbf{y}}

\newcommand{\R}{\mathbb{R}}
\newcommand{\E}{\mathbb{E}}

\newcommand*\diff{\mathop{}\!\mathrm{d}}

\newcommand{\bmu}{\bm{\mu}}
\newcommand{\etab}{\bm{\eta}}
\newcommand{\Amb}{\mathcal{X}}
\newcommand{\Man}{\mathcal{M}}

\newcommand{\curve}{\gamma}
\newcommand{\dcurve}{\dot{\gamma}}

\newcommand{\Loss}{\mathcal{L}}

\newcommand{\TiM}{\T_{\Id}\Man}
\newcommand{\TposeS}{\T_{\pose}\mathcal{S}}
\newcommand{\TxM}{\T_{\x}\Man}
\newcommand{\TeM}{\T_{\mathbf{e}}\Man}
\newcommand{\T}{\mathcal{T}}
\newcommand{\TM}{\mathcal{T}\mathcal{M}}
\newcommand{\grad}[1]{\mathrm{grad}#1}

\newcommand{\PM}{\mathcal{P}(\Man)}
\newcommand{\Unif}{\mathcal{U}}
\newcommand{\RDF}{\mathcal{S}}
\newcommand{\normal}{\mathcal{N}}
\newcommand{\lieq}{\mathfrak{h}}
\newcommand{\lieQ}{\mathfrak{h}^K}
\newcommand{\adj}{\mathfrak{ad}}

\newtheorem{lemma}{Lemma}
\newtheorem{prop}{Proposition}
\newtheorem{dfn}{Definition}

\crefname{eq}{eq}{eq}
\Crefname{Eq}{Eq}{Eq}
\crefname{thm}{theorem}{theorem}
\Crefname{Thm}{Theorem}{Theorem}
\crefname{prop}{Prop.}{Prop.}
\crefname{dfn}{Dfn.}{Dfn.}
\Crefname{Prop}{Proposition}{Proposition}
\crefname{remark}{remark}{remark}
\Crefname{Remark}{Remark}{Remark}
\Crefname{algorithm}{Alg.}{Alg.}

\DeclareMathOperator*{\argmin}{arg\min}

\newcommand\norm[1]{\lVert#1\rVert}

\renewcommand{\paragraph}[1]{{\vspace{0.6mm}\noindent \bf #1}.}

\newcommand{\newmodel}{NRDF}
\newcommand{\RDFGrad}{RDFGrad}
\newcommand{\newdist}{d_q}
\newcommand{\distnn}{d_{\mathrm{NN}}}
\newcommand{\fid}{FID}

\def\paperID{6591} 
\def\confName{CVPR}
\def\confYear{2024}

\title{NRDF: Neural Riemannian Distance Fields for Learning Articulated Pose Priors
}

\author{\begin{tabular}{ccccccccccccccc}\multicolumn{3}{c}{Yannan He \textsuperscript{1,2}} & \multicolumn{3}{c}{Garvita Tiwari \textsuperscript{1,2,3}} & \multicolumn{3}{c}{Tolga Birdal \textsuperscript{4}} & \multicolumn{3}{c}{Jan Eric Lenssen \textsuperscript{3}} &  \multicolumn{3}{c}{Gerard Pons-Moll \textsuperscript{1,2,3}}\end{tabular}\\\\
{\small \textsuperscript{1}University of Tübingen, Germany \qquad \textsuperscript{2} Tübingen AI Center, Germany}\\
{\small\textsuperscript{3}Max Planck Institute for Informatics, Saarland Informatics Campus, Germany}\\
{\small\textsuperscript{4}Imperial College London, United Kingdom}\\
{\small \url{https://virtualhumans.mpi-inf.mpg.de/nrdf}
}}

\begin{document}
\makeatletter
\let\@oldmaketitle\@maketitle
\renewcommand{\@maketitle}{
	\@oldmaketitle
		\begin{center}
\centering
    \captionsetup{type=figure}
    \vspace{-20pt}
    \includegraphics[width=0.95\textwidth]{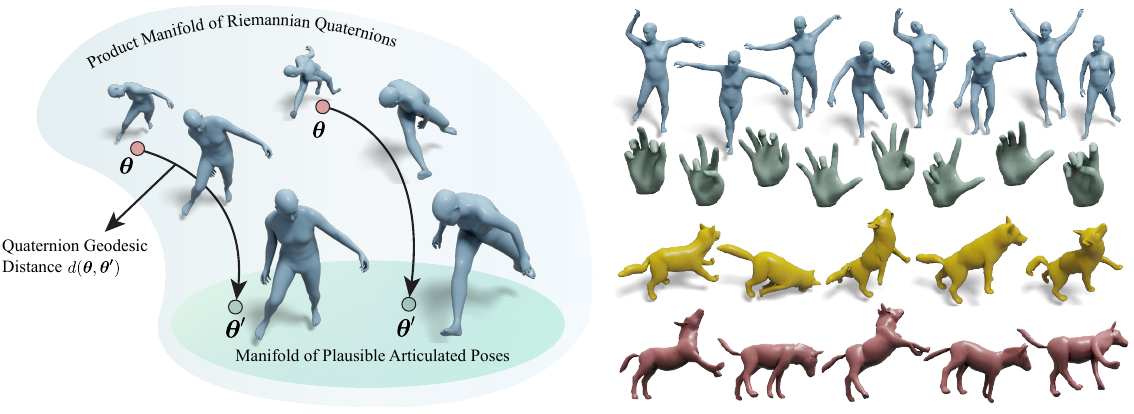}
    	\end{center}
     \vspace{-10pt}
\refstepcounter{figure}\normalfont Figure~\thefigure.  \emph{Left}: We present \textbf{Neural Riemannian Distance Fields (\newmodel{}s)}, a principled method to learn data-driven priors as subspace of high-dimensional Riemannian manifolds. \emph{Right}: \newmodel{}s can effectively model the pose of different articulated shapes. We present diverse samples generated using \newmodel{}s trained on human, hand, and animal poses respectively. 
	\label{fig:teaser}
	\newline
}
\makeatother
\maketitle

\begin{abstract}
Faithfully modeling the space of articulations is a crucial task that allows recovery and generation of realistic poses, and remains a notorious challenge.
To this end, we introduce Neural Riemannian Distance Fields (\newmodel{}s), data-driven priors modeling the space of plausible articulations, represented as the zero-level-set of a neural field in a high-dimensional product-quaternion space.
To train \newmodel{}s only on positive examples, we introduce a new \textbf{sampling algorithm}, ensuring that the geodesic distances follow a desired distribution, yielding a principled distance field learning paradigm.
We then devise a \textbf{projection algorithm} to map any random pose onto the level-set by an \textbf{adaptive-step Riemannian optimizer}, adhering to the product manifold of joint rotations at all times. 
\newmodel{}s can compute the Riemannian gradient via backpropagation and by mathematical analogy, are related to Riemannian flow matching, a recent generative model. 
We conduct a comprehensive evaluation of \newmodel{} against other pose priors in various downstream tasks, \emph{i.e.}, pose generation, image-based pose estimation, and solving inverse kinematics, highlighting \newmodel's superior performance. Besides humans, \newmodel's versatility extends to hand and animal poses, as it can effectively represent any articulation. 
\end{abstract}

\vspace{-5mm}
\section{Introduction}
\label{sec:introduction}
Pose and motion are ubiquitous yet very challenging and intriguing aspects of understanding articulated agents such as humans, animals or hands. Pose is intrinsic to the human experience, our interaction with each other and the environment. Understanding it is vital for applications in fields such as medicine, entertainment, AR/VR, \etc. As a result, human pose understanding, generation, and acquisition have been extensively studied in the domain of computer vision and graphics.

Acquisition using IMUs~\cite{HPS}, mocap markers~\cite{AMASS:2019}, and scans~\cite{dfaust} have accelerated the research direction by providing an enormous amount of data. These datasets are used to learn pose distributions, which are further used as priors in downstream tasks such as solving inverse-kinematics (IK), image HPS~\cite{SMPL-X:2019}, motion denoising, \etc. Previous work in this domain have used GMMs~\cite{bogo2016keep}, VAEs~\cite{SMPL-X:2019} and GANs~\cite{Davydov_2022_CVPR} to model pose prior. However, these methods are either limited by Gaussian assumptions~\cite{bogo2016keep,SMPL-X:2019} or risk suffering from instability of the training process~\cite{Davydov_2022_CVPR}. 

In this work, we propose \textbf{Neural Riemannian Distance Fields (\newmodel{}s)}, implicit, neural distance fields (NDFs)~\cite{Park_2019_CVPR,chibane2020ndf} constructed on the space of plausible and realistic articulations.
\newmodel s are induced by the geodesic distance on the product manifold of quaternions and are trained to predict the Riemannian distance.
In order to learn well-defined and detailed pose manifolds, we diligently study the role of training data distribution in learning distance field priors. To draw more samples near the surface with a gradual decrease for faraway regions~\cite{ShapeGF}, we introduce a \textbf{wrapped sampling algorithm on Riemannian manifolds} that allows explicit control over the resulting distance distribution. We show that heuristics developed in the past~\cite{tiwari22posendf} sample points around the surface and often do not lead to desired distribution characteristics.
The effect gets exacerbated in high-dimensional spaces like the product space of articulated bodies.

One of the key benefits of pose priors is the ability to map an arbitrary articulation onto a plausible one~\cite{tiwari22posendf}. To this end, we introduce an \textbf{adaptive-step Riemannian gradient descent} algorithm, \textbf{\RDFGrad}, in which the gradient obtained by a backward pass, scaled to the predicted distance, is used to update an articulated pose, respecting the product manifold of quaternions at all times. This is in stark contrast to Pose-NDF~\cite{tiwari22posendf}, which uses a Euclidean gradient descent to approximate the projection onto the articulation manifold. As a result, every projection step is followed by a re-projection onto joint rotations, resulting in slower convergence. Our manifold-aware formulation ensures the iterates to remain as articulated bodies.
Our models bear similarities to Riemannian Flow Matching (RFM)~\cite{chen2023riemannian,lipman2022flow}, the recent state-of-the-art generative models, as we explain. In fact, \newmodel~obtains the required gradients by backpropagation, whereas RFM is explicitly trained on them.

In summary, \textbf{our contributions are}:
\begin{itemize}
\item [$\bullet$] A principled framework for learning NDFs on Riemannian manifolds, with strong ties to flow matching
\item [$\bullet$] A theoretically sound adaptive-step Riemannian gradient descent algorithm that leads to accelerated convergence in mapping poses onto the learned manifold
\item [$\bullet$] A versatile framework for sampling training data, crucial for pose-manifold learning
\end{itemize}
The efficacy of \newmodel{} is shown in a range of downstream tasks, such as pose generation, solving inverse kinematics (IK) problems, and pose estimation from images. 
We observe that the \newmodel{}-based pose prior outperforms earlier works such as VPoser~\cite{SMPL-X:2019}, GAN-S~\cite{Davydov_2022_CVPR}, GFPose~\cite{ci2022gfpose}, Pose-NDF~\cite{tiwari22posendf} on aforementioned tasks, under pose distance metrics. We also conduct a user study about the perceptual quality of different pose distance metrics. As \newmodel{} can be easily applied to any articulated shape, we also evaluate our model on hand poses and animal poses.

\vspace{-2mm}
\section{Related Work}
\vspace{-1.5mm}
We now review articulated (\eg human) pose and motion priors crucial for understanding human pose from images~\cite{bogo2016keep,hmrKanazawa17,SPIN:ICCV:2019,bhatnagar2019mgn}, videos~\cite{vibe,stoll_videobased}, IMUs~\cite{HPS,bhatnagar22behave,vonMarcard2018} and scans~\cite{dfaust,alldieck19cvpr}.
As we explore the connection to flow-matching models, we also review related literature therein.

\paragraph{Early pose priors}
Initial works in modeling robust pose prior learn constraints for joint limits in Euler angles~\cite{euler_jointlimit} or swing and twist representations~\cite{range_ballsocket_twist,shao_realisitic,akhter_jointlimits}. However, these methods mainly rely on small-scale datasets and can still produce unrealistic poses due to unreal combinations of different joints. This was followed by more sophisticated models such as GMMs~\cite{bogo2016keep} or PCAs~\cite{pca_cyclic,3dpeople_gaussian,Black:ECCV:2000}. 

\paragraph{VAEs and GANs}
Recent generative deep learning models 
have harnessed large-scale datasets to train VAEs~\cite{SMPL-X:2019,rempe2021humor,zhang2021learning,ACTOR:ICCV:2021} and GANs~\cite{Davydov_2022_CVPR,hmrKanazawa17,BarsoumCVPRW2018} as pose/motion priors, either in a task-dependent~\cite{hmrKanazawa17,BarsoumCVPRW2018} or task-independent~\cite{Davydov_2022_CVPR,SMPL-X:2019} manner. For instance, \cite{hmrKanazawa17} trains a GAN for $p(\pose|I)$ for image-based pose estimation, HP-GAN~\cite{BarsoumCVPRW2018} models $p(\pose_t|\pose_{t-1})$, representing the current pose given the previous one. HuMoR~\cite{rempe2021humor} proposes to learn a distribution of possible pose transitions in motion sequences using a conditional VAE. ACTOR~\cite{ACTOR:ICCV:2021} learns an action-conditioned VAE-Transformer prior. On the other hand, more task-independent models such as VPoser~\cite{SMPL-X:2019} learn a VAE using AMASS~\cite{AMASS:2019} dataset. However, because of Gaussian assumptions in the latent space, the model is biased towards generating mean poses, along with the risk of generating unrealistic poses from dead-regions of Gaussians~\cite{tiwari22posendf,Davydov_2022_CVPR}.~\cite{Davydov_2022_CVPR} learn a human pose prior using GANs, overcoming the limitations of Gaussians, but requires training a GAN, which is known to be unstable~\cite{NIPS2016_8a3363ab}.

\paragraph{Probabilistic flow}
More recently, pose and motion prior have also been developed using widely popular diffusion-based models~\cite{song2021scorebased,song_gradient,ho2020denoising}. GFPose learns score function (gradient of log-likelihood) of a task conditional distribution.
Likewise, MDM~\cite{tevet2023human} and MoFusion~\cite{dabral2022mofusion} model motion sequences conditioned on tasks through a diffusion process. Normalizing flows~\cite{Kobyzev_2021} have been applied in human pose-related tasks for a while, to address the ambiguous inverse 2D-to-3D problem~\cite{WehRud2021}, or recently to perform anomaly detection~\cite{hirschorn2023normalizing}. Flow Matching was also introduced to Riemannian manifolds~\cite{chen2023riemannian}. We show that our method has strong ties to the flow matching principle and apply it to pose for the first time.

\paragraph{Distance fields}
Closest to our work is Pose-NDF~\cite{tiwari22posendf}, which also models the manifold of plausible human poses using neural distance fields. Pose-NDF uses the learned distance field and its gradient to project arbitrary poses onto a manifold, using Euclidean gradient descent, where every step is followed by a re-projection onto the $SO(3)$. This results in slower convergence. In contrast, we leverage  an adaptive-step Riemannian gradient descent which ensures that the iterates always remain on $SO(3)^K$, yielding faster convergence. Moreover, Pose-NDF's training data generation is naively engineered, requiring a difficult per-task fine-tuning. We introduce a novel sampling method based on recent advances in scheduled optimal transport sampling~\cite{chen2023riemannian}, to create training data which results in robust learning without the need of manual, task-specific tuning.

\section{Background}
\label{sec:background}
We first introduce the necessary preliminaries to define our Riemannian distance fields.
Following~\cite{birdal2018bayesian,birdal2019probabilistic,chen2022projective}, we define an $m$-dimensional \textit{Riemannian manifold}, embedded in an ambient Euclidean space $\Amb = \R^d$ and endowed with a \textit{Riemannian metric} $\G\triangleq (\Gx)_{\x\in\Man}$ to be a smooth curved space $(\Man,\G)$. A vector $\v\in\Amb$ is said to be \emph{tangent} to $\Man$ at $\x$ \emph{iff} there exists a smooth curve $\curve:[0,1]\to\Man$ s.t. $\curve(0)=\x$ and $\dcurve(0)=\v$. The velocities of all such curves through $\x$ form the \emph{tangent space} $\TxM=\{ \dcurve (0) \,|\, \curve:\R\to\Man \text{ is smooth around $0$ and } \curve(0)=\x\}$, whose union is called the \emph{tangent bundle}: $\TM=\bigcup_{\x}\TxM=\{(\x,\v)\mid\x\in\Man, \v\in\TxM\}$. 
The Riemannian  metric $G(\cdot)$ equips each point $\x$ with an inner product in the tangent space $\TxM$, $\langle\ub, \vb\rangle_x = \ub^T\Gx\vb$.
We will also work with a product of $K$ manifolds, $\Man_{1:K}\vcentcolon=\Man_1\times\Man_2\times\dots\times\Man_K$, For identical manifolds, \ie $\Man_i\equiv\Man_j$, we recover the \emph{power manifold}, $\Man^K\vcentcolon=\Man_{1:K}$, whose tangent bundle admits the \emph{natural isomorphism}, $\mathcal{T}\Man^K\simeq(\TM\times\dots\times\TM)$.
We now define the operators required for our algorithm.
\vspace{-1mm}
\begin{dfn}[Riemannian Gradient]
For a smooth function $f:\Man\to\R$ and $\forall (\x,\v)\in \TM$, 
we define the \emph{Riemannian gradient} of $f$ as the unique vector field $\grad{f}$ satisfying~\cite{boumal2020introduction}:
\vspace{-1mm}
\begin{equation}
    \mathrm{D}f(\x)[\v] = \langle \v, \grad{f(\x)} \rangle_\x
\vspace{-1mm}
\end{equation}
where $\mathrm{D}f(\x)[\v]$ is the derivation of $f$ by $\v$. It can further be shown (see our supplementary) that an expression for $\grad{f}$ can be obtained through the projection of the \emph{Euclidean} gradient orthogonally onto the tangent space
\vspace{-1mm}
\begin{equation}\label{eq:gradf}
\grad{f(\x)} = \nabla f(\x)_{\|} = \Pi_{\x}\big( \nabla f(\x)\big).    
\vspace{-1mm}
\end{equation}
where $\Pi_{\x}:\Amb\to\TxM\subseteq \Amb$ is an orthogonal projector with respect to $\langle \cdot,\cdot \rangle_{\x}$. 
\end{dfn}
In most packages such as ManOpt~\cite{pymanopt},~\cref{eq:gradf} is known as the \emph{egrad2rgrad}.
\vspace{-1mm}
\begin{dfn}[Riemannian Optimization]
We consider gradient descent to solve the problems of $\min_{\x\in\Man}f(\x)$. For a local minimizer or a \emph{stationary point} $\x^\star$ of $f$, the Riemannian gradient vanishes $\grad{f(\x^\star)}=0$ enabling a simple algorithm, \emph{Riemannian Gradient Descent} (RGD):
\vspace{-0.5\baselineskip}
\begin{align}\label{eq:RGD}
\x_{k+1} = R_{\x_k}(-\tau_k\,\grad{f(\x_k)})
\end{align}

where $\tau_k$ is the step size at iteration $k$ and $R_{\x_k}$ is the \emph{retraction} usually chosen related to the exponential map. 
Note that both RGD and its stochastic variant~\cite{bonnabel2013stochastic} are practically convergent~\cite{boumal2020introduction,munier2007steepest,zhang2016first,bonnabel2013stochastic,tripuraneni2018averaging}. Though, only in rare cases is $\tau_k$ analytically computable. Therefore, most minimizers use either Armijo or Wolfe line-search ~\cite{absil2009accelerated}.
\end{dfn}

\paragraph{Quaternions $\Quats$}
A unit quaternion $\q\in\Quats$ represents a rotation using a 4D unit vector $[w\vcentcolon= q_1, \vb\vcentcolon= (q_2,q_3,q_4)]$ double covering the non-Euclidean $3$-sphere, \ie, $\q\equiv-\q$ identify the same rotation. The inverse or \emph{conjugate} of $\q$ is given by $\bar{\q}\vcentcolon=\q^{-1}=(w,-\vb)$, whereas
the non-commutative multiplication of two quaternions $\q=(q_1, \vb_q)$ and $\rb=(r_1, \vb_r)$ is defined to be $\q\otimes\rb \vcentcolon=\q\rb\vcentcolon=
({q}_1{r}_1-\vb_p\cdot \vb_r,\,{p}_1\vb_r+{r}_1 \vb_p+\vb_p \times \vb_r)$.
Following~\cite{angulo2014riemannian,birdal2020synchronizing}, we now briefly explain the Lie group structure of the quaternions essential for manifold optimization.
\begin{dfn}[Exponential map]
The \emph{exponential map} $\Exp_{\q}(\cdot)$ maps any vector in $\TqM$ onto $\Quats$: 
\vspace{-0.5\baselineskip}
\begin{align}
    \Exp_{\q}(\etab)=\q\exp(\etab) = \q\Big( \cos(\theta), \vb \frac{\sin(\theta)}{\theta}\Big),
\end{align}
where $\etab=(w, \vb)\in \TqM$ and $\theta = \| \vb \|$.
\end{dfn}
\begin{dfn}[Logarithmic map]
The inverse of $\emph{exp-map}$, $\Log_{\q}(\bp) : \Quats \to \TqM$ is $\emph{log-map}$ and defined as:
\vspace{-0.5\baselineskip}
\begin{align}
    \Log_{\q}(\bp)=\log(\q^{-1}\bp) = \Big(0, \frac{\vb}{\|\vb\|} \arccos(w) \Big),
\end{align}
\vspace{-0.5\baselineskip}
this time with a slight abuse of notation $\q^{-1}\bp=(w, \vb)$. 
\end{dfn}
\begin{dfn}[Quaternion geodesic distance ($\newdist$)]
\label{dfn:riemann_dist}
Let us rephrase the Riemannian distance between two unit quaternions using the logarithmic map, whose norm is the length of the shortest geodesic path. Respecting the antipodality:
\begin{align}
    d(\q_1,\q_2) = \begin{cases}
    \| \Log_{\q_1}(\q_2) \| \,\,\,\,=  \arccos(w), & w\geq 0 \\
    \| \Log_{\q_1}(-\q_2) \| =\arccos(-w), & w<0
    \end{cases}\nonumber
\end{align}
where $\q_1^{-1}\q_2 = (w,\vb)$.
\end{dfn}

\section{Neural Riemannian Distance Fields}

We start by explaining \emph{Riemannian Distance Fields} to model realistic articulated shapes in~\secref{sec:model} and introduce our novel projection algorithm to map onto this space while adhering to the manifold of joint rotations. We then propose \newmodel~and a novel method for sampling articulated poses, to generate desired training data, in~\secref{sec:nrdf}. We conclude this section by forming a link between recently-popularized flow matching models~\cite{chen2023riemannian} and \newmodel.

\subsection{Modeling of Plausible Articulated Poses}
\label{sec:model}
We parameterize the pose of a 3D articulated body composed of $K$ joints, $\pose\vcentcolon=\{\q_i\in\Quats\}_{i=1}^K$, on the power manifold of quaternions $\Quats^K=\Quats\times\dots\times\Quats$. 
\begin{dfn}[Geometry of 3D articulated poses]\label{dfn:humanpose}
    $\Quats^K$ turns into a Riemannian manifold $(\Quats^K,\G^K)$ when endowed with the $L_p$ \emph{product metric} $d_{\Quats^K}:\Quats^K\times\Quats^K\to\R$:
    \begin{align}
        d_{\Quats^K}(\pose,\pose^\prime)=\| d(\q_1,\q_1^\prime),d(\q_2,\q_2^\prime),\dots,d(\q_K,\q_K^\prime)\|_p,\nonumber
    \end{align}
    where $\q\in\pose\in\Quats^K$ and $\q^\prime\in\pose^\prime\in\Quats^K$. In this work, we use $p=1$. The natural isomorphism further allows us to write its exponential map $\Exp_{\pose}:\T{\Quats^K}\to\Quats^K$ component-wise: $\Exp_{\pose}=\left(\Exp_{\q_1},\Exp_{\q_2},\dots,\Exp_{\q_K}\right)$. Akin to this, is the logarithmic map, $\Log_{\pose}$. 
    Since the tangent spaces and therefore $\Pi_{\pose}$ are replicas, the gradient of a smooth function $f:\Quats^K\to\R$ w.r.t. $\pose$ is also the Cartesian product of the individual gradients:
    \begin{align}
        \grad_{\pose}{f(\pose)} = \left(\grad_{\q_1}{f(\pose)},\,\dots,\,\grad_{\q_K}{f(\pose)}\right).
    \end{align}
\end{dfn}

\begin{dfn}[Riemannian Distance Fields (RDFs)]
    Given the parameterization above, we model the manifold of \emph{realistic and plausible articulations} (as defined by a dataset) on the zero level set of a model $f_\phi: \Quats^K \to \R^+$:
    \begin{equation}
\mathcal{S} = \{\pose \in \Quats^K \mid f_\phi(\pose) = 0\} \textrm{,}
\end{equation}
such that the value of $f$ represents the unsigned geodesic distance to the closest plausible pose on the manifold. 
\end{dfn}

\begin{prop}[\RDFGrad{}]\label{prop:RDFGrad}
Given $\RDF$ (hence $f$), we employ an adaptive-step Riemannian optimizer, to project any pose $\pose_0$ onto the plausible poses:
\begin{equation}\label{eq:rie_proj}
\pose_{k+1} = \mathrm{Exp}_{\pose_k}\left(-\alpha f(\pose_k)\,\frac{\grad{f(\pose_k)}}{\|\grad{f(\pose_k)}\|}\right)\textnormal{.}
\end{equation}
\end{prop}
The details are given in~\cref{sec:background}. This procedure is in contrast to Pose-NDF~\cite{tiwari22posendf}, which uses a \emph{projected} (Euclidean) gradient descent to approximate the projection onto the manifold.
We therefore require the expression for projecting onto the tangent space of a quaternion, whose explicit form seems to be lacking in the literature. In what follows, we derive this operator.
\begin{prop}[Quaternion-egrad2rgrad]\label{prop:egrad2rgrad}
For the quaternion manifold, the projection and mapping onto the tangent space of the canonical unit quaternion $\eb=\begin{bmatrix} 1 & 0 & 0 & 0 \end{bmatrix}^\top$ (egrad2rgrad in~\cref{eq:gradf}) takes the form:
\begin{align}\label{eq:egrad2rgrad}
    \Pi_{\q}(\vb) &= \Pm\vb - \frac{\eb^\top \Pm\vb}{1+\q^\top \eb} (\q+\eb)\\
    &= \begin{bmatrix}
        0 &  0 & 0 & 0\\
        -q_2/(1+q_1) &  1 & 0 & 0\\
        -q_3/(1+q_1) &  0 & 1 & 0\\
        -q_4/(1+q_1) &  0 & 0 & 1
    \end{bmatrix}\Pm\vb \textnormal{,}
\end{align}
where $\vb\in\R^4$ and $\Pm(\q)=\Id-\q\q^\top$.

\end{prop}
\begin{proof}[Sketch of the proof]
    The full proof uses the projection operator of ${\mathcal{S}}^3$ as well as the \emph{parallel transport} of the quaternion manifold. We leave the full proof to our supplementary.
\end{proof}

\subsection{Learning RDFs}
\label{sec:nrdf}
We now describe how we construct $\RDF$, \ie, learn $f_{\phi}$.
\begin{dfn}[Neural RDFs (\newmodel{}s)]
    We model $f$ using a combination of hierarchical network and an MLP decoder, similar to Pose-NDF~\cite{tiwari22posendf}. Given a dataset $\mathcal{D} = \{\pose_i\}_{1\leq i \leq N}$ of articulated poses and a scheduled sampler for network inputs, the network is trained to predict the distance to the closest example from dataset $\mathcal{D}$: 
\begin{equation}
\label{eq:training}
{\phi^\star} = \argmin_{\phi} \sum\limits_{i=1}^N \norm{ f_\phi(\pose_i) - \min_{\pose^\prime\in\mathcal{D}} d(\pose_i, \pose^\prime)} \textnormal{.}
\end{equation}
We call $f_{\phi^\star}$, learned in this way, an \emph{NRDF}. We obtain $\grad f_{\phi^\star}(\pose)$ via backprop. followed by an \emph{egrad2rgrad}.
\end{dfn}
\vspace{-2mm}

\paragraph{Sampling training data}
While the positive examples, lying on $\RDF$ are provided, the model performance strongly depends upon the statistical distribution of training examples and including  sensible negative samples ($d>0$) is critical.
This is also observed in the training of general neural distance fields for tasks like 3D shape reconstruction or completion~\cite{chibane20ifnet,ShapeGF,tiwari21neuralgif}. To effectively capture intricate details of the pose manifold, it is essential to have an abundance of samples in proximity to the pose manifold $d<\epsilon$, gradually decreasing as we move away from it. This ensures that the network sees data points spanning the entire space, resulting in a well-behaved and continuous learned distance field. 

Pose-NDF~\cite{tiwari22posendf} samples a training pose as $\frac{\pose+\epsilon}{\norm{\pose+\epsilon}_2}$, where $\pose \sim \mathcal{D}$ and $\epsilon \sim \normal(0,\sigma\mathbf{I})\in\R^{4K}$.
We observe the following: \textbf{(1)} This specific sampling technique leads to distances, which are roughly $\mathcal{X}$-distributed for large $k$ before projection, as shown in Fig.~\ref{fig:distance_distr}{\color{red}a}, as the distance is the sqrt-sum of squared Normal distributions with variance $\sigma$. This is contrary to the goal that the data should contain more samples close to the manifold. \textbf{(2)} Simply corrupting data samples by Euclidean noise does not expose explicit control over the distribution of generated distances, complicating the design of a schedule adhering to distance-related conditions. 

In the following, we propose a framework for data sampling that allows for explicit control over generated distance distributions. As outlined in~\cref{alg:sampledata_jan}, given an arbitrary distribution $\mathcal{P}$ over $\mathbb{R}^+$ and an input pose $\pose$, the algorithm first samples a distance $h \in \mathbb{R}$ and then generates an example, $h$ apart from $\pose$. It does so by sampling a direction $\mathbf{v} \in \mathcal{T}_{\pose}\Man^K$ independently from $h \in \mathcal{P}$, before finding the new pose via interpolation in tangent space.
We can now show that the distances $h$ sampled this way translate to the examples.

\begin{algorithm}[t]
\caption{Sampling in $SO(3)$ for articulated poses}
    \begin{algorithmic}[1]
    \Require Data example $\pose$, distribution $\mathcal{P}$
    \Ensure A pair $(\hat{\pose}, h)$, input to the network.

    \hspace{-1.38cm} Sample distance from arbitrary $\mathcal{P}$:
    \State $h \sim \mathcal{P}$, $h \in \mathbb{R}^+$

    \hspace{-1.3cm}Sample direction $\mathbf{v}$ uniformly from unit sphere in $\mathcal{T}_{\pose}\Quats^K$:
     \State $\mathbf{v} \sim \mathcal{N}_{\mathcal{T}_{\pose}\Quats^K}(\vec{0}, \vec{1}), \mathbf{v} \in \mathcal{T}_{\pose}\Quats^K$
    \State $\mathbf{v} \leftarrow {\mathbf{v}}/{\norm{\mathbf{v}}}$

    \hspace{-1.3cm}Interpolate in $\mathcal{T}_{\pose}\Quats^K$ and map to $\Quats^K$:
    \State  $ \hat{\pose} \leftarrow\text{Exp}_{\pose}(h \mathbf{v})$ 
    \end{algorithmic}
    \label{alg:sampledata_jan} 
\end{algorithm}
\setlength{\textfloatsep}{5pt}
\vspace{-1mm}
\begin{prop}[Distance preservation]\label{prop:distance_pre}
Let $\mathcal{P}$ be a distribution over domain $[0,1]$, $\pose \in \mathcal{D}$ a data example, $\hat{\pose} \in \Quats^K$ the output of Alg~\ref{alg:sampledata_jan} with input $(\pose, \mathcal{P})$, and $d= d(\pose, \hat{\pose})$. Then, for the distribution of resulting distances holds $p(d) = \mathcal{P}$.

\end{prop}
\begin{proof}[Sketch of the proof]
    The proof uses the distance preservation of logarithmic and exponential maps in the base. A full proof is given in the supplemental materials.
\end{proof}

Given the introduced framework, we can induce a distribution $\mathcal{P}$, e.g. half-Gaussian, or exponential, as shown in Fig~\ref{fig:distance_distr}. Note that~\cref{prop:distance_pre} holds for distances to the seed example $\mathbf{p}$, while we show the distribution of distances to the closest neighbor in $\mathcal{D}$ (c.f.~\cref{eq:training}). Thus, we observe slight distribution shifts to the left.

\paragraph{Sampling diverse poses} 
To generate diverse pose samples on the manifold, we adopt an iterative procedure. We use~\cref{alg:sampledata_jan} to produce an initial sample and then project it onto the zero level set via the proposed RDFGrad. 

\paragraph{\newmodel~and Riemannian Flow Matching (RFM)~\cite{chen2023riemannian}}
RFM is a \emph{simulation-free} method for learning continuous normalizing flows (CNFs)~\cite{chen2018neural} on Riemannian manifolds, finding the \emph{optimal transport} (OT) between a simple distribution and the data distribution. Interestingly, we can make a strong connection between our framework and flow matching. While flow matching predicts steps along OT trajectories towards the data manifold via feed-forward prediction, we find these steps as a gradient of our distance field via backpropagation. Our data generation procedure additionally ensures that (1) the $t\in[0,1]$ of flow matching is a scaled variant of distance (by normalization of $\mathbf{v}$ in \cref{alg:sampledata_jan}), and (2) we recompute nearest neighbors from $\mathcal{D}$ after sample generation. We provide a formal connection in the supplementals. In general, our distance field formulation has some advantages: instead of predicting the step towards the manifold via an autoencoder, we obtain it via backpropagation. This allows optimization in domains in which designing decoders is challenging. Also, we can utilize existing optimizers of deep learning frameworks for Lagrangian iterations by simply minimizing distance.

\vspace{-2mm}
\section{Experiments and Results}
\vspace{-1mm}
In this section, we evaluate the performance of \newmodel{} on a range of downstream tasks and provide a comparison with baselines and prior work. \newmodel{} incorporates three key components: an innovative sampling method for training data generation (Alg.\ref{alg:sampledata_jan}), Riemannian distance ($\newdist{}$ from Def~\ref{dfn:riemann_dist}), and a novel projection using \RDFGrad{}. Ablation studies for each component are detailed in~\secref{exp:ablations}, alongside a comparison with Pose-NDF~\cite{tiwari22posendf}. We also compare our model to score-based model~\cite{ci2022gfpose} and RFM~\cite{chen2023riemannian}-based work, emphasizing the strong mathematical connection between the latter and \newmodel. We demonstrate the application of \newmodel{} as a prior in downstream tasks, such as pose generation~(\secref{exp:generation}), solving IK from sparse/partial observations~(\secref{exp:iksolver}), and for human pose estimation from images~(\secref{exp:hpsimages}). We compare against previous pose priors such as VPoser~\cite{SMPL-X:2019}, Pose-NDF~\cite{tiwari22posendf}, GFPose~\cite{ci2022gfpose}, GAN-S~\cite{Davydov_2022_CVPR} and our own baselines. Since \newmodel{} can be easily extended to any articulated shape, we show results on hand and animal poses in~\secref{exp:animalhand}.

\begin{figure}[t]
\includegraphics[width=\linewidth]{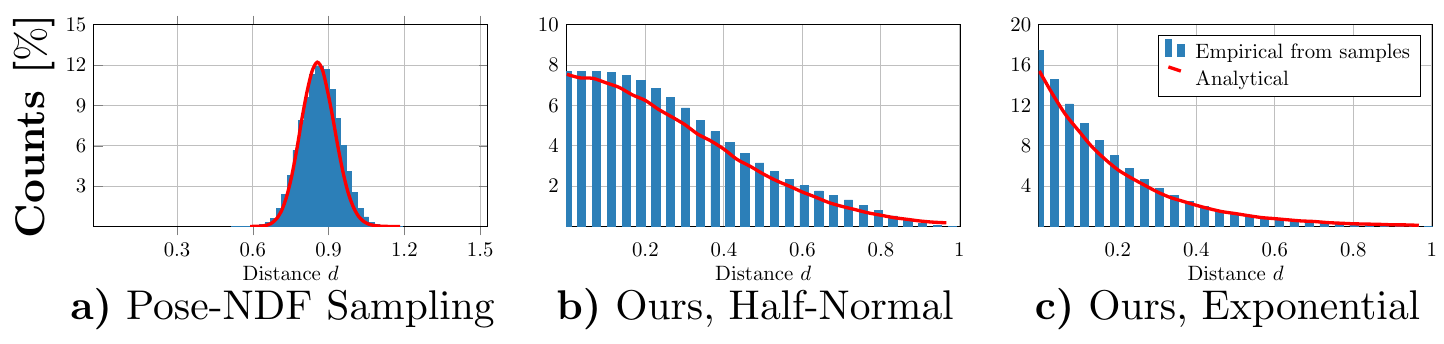}\vspace{-3mm}
\caption{\textbf{Distance distributions and histograms from different sampling strategies}. \textbf{a)} Pose-NDF sampling generates a $\mathcal{X}$-like distribution for large $k$, which does not fit the needs of distance field learning. Our sampling schedule allows to control the distance distribution, e.g. to follow \textbf{b)} Half-Gaussian or \textbf{c)} Exponential distributions. The histograms show the distance to the closest example in $\mathcal{D}$, not the distance to the original example, resulting in a slight distribution shift to the left due to neighbor changes with increasing distance.}
\label{fig:distance_distr}
\end{figure}

\paragraph{Evaluation metric}
Previous research~\cite{posemetric_GPD,posemetric_localfast} highlighted a significant gap in standard distance metrics based on joint locations, orientations, and user perception. We conducted a user study and identified that the most user-perceptive metric combines quaternion in the global frame and Euclidean marker-to-marker distance ($\Delta\q$ + m2m), followed by only marker-to-marker distance (m2m). These metrics are used for evaluation in our experiments, with user study details in the supplemental materials. To assess pose generation diversity, we utilize Average Pairwise Distance (APD)~\cite{stoch_diverse}, and for realism evaluation, we employ Fréchet distance (\fid{}) and a distance metric $\distnn = \min_{\pose'\in\mathcal{D}} \newdist{}(\pose, \pose')$. \fid{} gauges the similarity between the distributions of real and generated samples, while $\distnn$ measures the distance between the generated pose and its nearest neighbor from the training data.

\paragraph{Baselines} 
Here, we present baselines inspired by flow matching and diffusion models. Specifically, we adopt a score-based human pose model, GFPose~\cite{ci2022gfpose}. However, there are two key differences in the experimental setup: 1) GFPose is modeled using joint-location representation, and 2) the original model is trained on H3.6M~\cite{h36m_pami} dataset. We implement two baselines based on GFPose: \textbf{GFPose-A} and \textbf{GFPose-Q}. GFPose-A, is trained using the AMASS dataset, in joint-location representation, and GFPose-Q, is trained using quaternion representations. We also introduce two RFM~\cite{chen2023riemannian}-based baselines, namely \textbf{FM-Dis} and \textbf{FM-Grad}. FM-Grad represents the original RFM model trained for pose-denoising tasks with time conditioning, while FM-Dis closely aligns with our approach, where we predict the distance (without $t$-conditioning) and obtain the vector field toward the manifold through backpropagation in contrast to the direct prediction of the vector field in FM-Grad. 
Models based on distance fields utilize gradients calculated through backpropagation to approach the manifold. Additionally, we implement a Gradient Prediction network, which directly predicts this gradient, and unlike diffusion, this approach is not time-conditioned. We provide more details in the supplementary. 
We apply the aforementioned baselines for pose denoising and generation. However,  only FM-Dis is employed for optimization tasks, as others are formulated using the  $t$-conditioned model, rendering them unsuitable for integration into the optimization pipeline.

\begin{table}[t]
\centering
\resizebox{\columnwidth}{!}{
\begin{tabular}{lcccccc}
\toprule
Method &  Avg. Conv. Step$\downarrow$ & $\Delta\q$ +m2m $\downarrow$ & m2m (cm) $\downarrow$  \\
\midrule
Pose-NDF~\cite{tiwari22posendf}  &   40 & 0.349 & 25.04  \\
GFPose-Q & $\backslash$ & 0.359 & 24.43 \\
Gradient Prediction w/o time  & 68 &  0.401 & 37.26  \\
FM-Dis (w/o time) & 33 & 0.230 & 18.74 \\
FM-Grad (w/ time) & 52 & 0.216 & 16.72 \\
\midrule
\textbf{Ours} ($\alpha$=1.0)  &   \textbf{8} & \textbf{0.170} & \textbf{14.32}  \\
Ours ($\alpha$=0.5)  &   34 & 0.180 & 15.05  \\
Ours ($\alpha$=0.01)  &   34 & 0.201 & 16.59  \\
\midrule
Ours w/o \RDFGrad{}  &  48 & 0.171 & 14.51   \\
Ours w/o (\RDFGrad{}, $\newdist{}$)  &   100 & 0.179 & 15.15  \\
\bottomrule
\end{tabular}
}\vspace{-3mm}
\caption{\label{tab:ablation} \textbf{Comparison with baselines and model ablations on pose denoising:} We evaluate ($\Delta\q$ +m2m) and m2m between denoised poses and their ground truth nearest neighbors. Our method achieves the best accuracy while converging faster, thanks to the novel training data generation, $\newdist{}$ and \RDFGrad. \vspace{-1mm}}
\end{table}
\vspace{-1mm}
\subsection{Comparison with Baselines and Ablation Study}
\vspace{-2mm}
\label{exp:ablations}
We first compare our model with prior works on the task of pose denoising in~\tabref{tab:ablation} (top) and evaluate using ($\Delta\q$ +m2m) and m2m. In contrast to the prior distance field-based model Pose-NDF~\cite{tiwari22posendf}, \newmodel{} exhibits significantly lower error in the pose denoising task. We attribute this improvement to three key components of our model: training data generation, a Riemannian distance metric $\newdist{}$, and \RDFGrad{}-based projection. Specifically, unlike Pose-NDF, our training involves the sampling of more poses near the manifold, gradually decreasing as we move away from it. This method results in well-behaved training data, contributing to a continuous and more accurate learned manifold. Furthermore, the new \RDFGrad{}-based projection ensures that the projection adheres to the manifold of poses, eliminating the need for re-projection as seen in Pose-NDF. Consequently, the convergence is faster, as indicated in~\tabref{tab:ablation}. We note that GFPose-Q exhibits high error, primarily because denoising with GFPose-Q collapses to a mean pose. This suggests that training a diffusion/score-based model on quaternion representation poses challenges.
Similar behavior is observed in FM-Dis, which tends to generate more common poses. The Gradient Prediction Network yields very high error, emphasizing the difficulty of training a gradient without time-conditioning. FM-Grad, which predicts flow with time conditioning, performs slightly better than the Gradient Prediction Network but still lags behind \newmodel{} in terms of performance.

\paragraph{Model ablation} We now perform ablation on each component of our model, including training data generation, Riemannian distance metric $\newdist{}$, and \RDFGrad{}-based projection. show results in~\tabref{tab:ablation}(bottom). \emph{Ours w/o \RDFGrad{} }, yields similar error rates but exhibits slower convergence speed, indicating that \RDFGrad{}-based projection adheres to the manifold and facilitates faster convergence. \emph{Ours w/o} (\RDFGrad{},  $\newdist{}$) results in decreased accuracy, emphasizing that the new distance metric contributes to more accurate predictions. Finally, Pose-NDF is \emph{Ours w/o} (\RDFGrad{},  $\newdist{}$) and w/o new training data, and we observe that the performance degrades significantly.
\vspace{-2mm}

\begin{figure*}[t]
    \centering
    \includegraphics[width=1.0\textwidth]{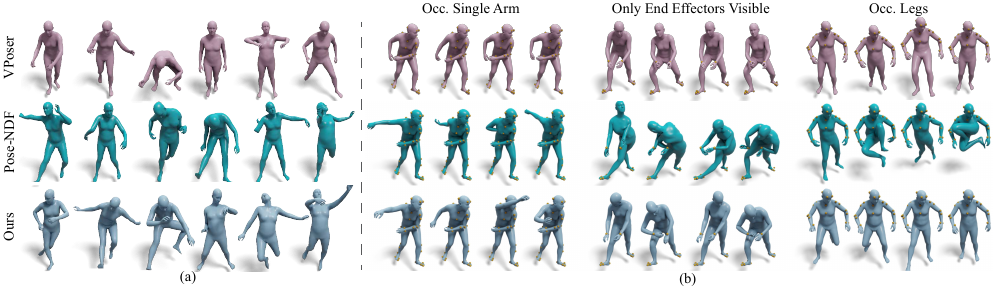}
    \vspace{-7mm}
    \caption{\textbf{(a) Pose generation:} \textbf{VPoser} generates realistic but somewhat limited diverse poses. \textbf{Pose-NDF} generates highly diverse poses but tends to yield unrealistic results (e.g., the third pose). \textbf{\newmodel{}} demonstrates a balance between diverse and realistic poses. \textbf{(b) IK Solver from partial/sparse markers:} Given partial observation (yellow markers), we perform 3D pose completion. We observe that VPoser~\cite{SMPL-X:2019} based optimization generates realistic, yet fixed and less diverse poses. Pose-NDF~\cite{tiwari22posendf} generates more diverse, but sometimes unrealistic poses, especially in case of very sparse observations. \newmodel{} generates diverse and realistic poses in all setups.\vspace{-6mm}}
	\label{fig:partialIK}
\end{figure*}
\subsection{Diverse Pose Generation}\vspace{-1mm}
\label{exp:generation}
We compare \newmodel{} for pose generation with classical GMM~\cite{bogo2016keep,omran2018neural}, VPoser~\cite{SMPL-X:2019}, GAN-S~\cite{Davydov_2022_CVPR}, Pose-NDF~\cite{tiwari22posendf}, diffusion-based pose priors (GFPose-A, GFPose-Q), and RFM~\cite{chen2023riemannian}-based model (FM-Dis), presenting the results in~\tabref{tab:posegen}. Evaluating realism, VPoser shows the lowest \fid{} and $\distnn$, indicating similarity to training samples, while Pose-NDF exhibits high \fid{} and $\distnn{}$, suggesting more divergence from training samples but resulting in unrealistic poses. Meanwhile, \fid{} and $\distnn{}$ for \newmodel{} are higher than VPoser and GAN-S, but lower than Pose-NDF. \newmodel{} strikes a balance, producing diverse yet realistic poses. 

Assessing pose diversity using APD reveals Pose-NDF with the highest values and VPoser with significantly lower scores. \newmodel{} shows substantial APD, indicating more diversity than VPoser but less than Pose-NDF. However, given the large \fid{} values, Pose-NDF tends to produce unrealistic poses, evident in~\figref{fig:partialIK} (a). VPoser, while less diverse, maintains realism. \newmodel{} proves to be a good tradeoff between diversity and realism, both visually and numerically. For the remaining priors, we note that they exhibit much lower APD, primarily generating mean poses.

\begin{table}[t]
\centering
 \small
\resizebox{\columnwidth}{!}{
\begin{tabular}{lcccccc}
\toprule
Method & \fid{} $\downarrow$  &  APD$\uparrow$\ (in cm) & $\distnn$ $\downarrow$ (in rad) \\

\midrule
GMM~\cite{bogo2016keep}  & $0.435^{\pm.017}$ & $21.944^{\pm.102}$ & $0.159^{\pm.001}$ \\
VPoser~\cite{SMPL-X:2019} &   $0.048^{\pm.002}$ & $14.684^{\pm.138}$ & $0.074^{\pm.000}$ \\
GAN-S~\cite{Davydov_2022_CVPR} &  $0.201^{\pm.030}$ & $10.914^{\pm.396}$ & $0.098^{\pm.001}$ \\
Pose-NDF~\cite{tiwari22posendf} &   $3.920^{\pm.034}$ & $37.813^{\pm.085}$ & $0.838^{\pm.001}$ \\
GFPose-A &   $1.246^{\pm.005}$ & $13.876^{\pm.116}$ & $\backslash$ \\
GFPose-Q &   $1.624^{\pm.002}$ & $6.773^{\pm.112}$  & $0.159^{\pm.000}$ \\
FM-Dis  &   $0.346^{\pm.007}$ & $6.849^{\pm.199}$ & $0.086^{\pm.001}$ \\
Ours ($\alpha$=0.01) &   $0.636^{\pm.007}$ & $23.116^{\pm.105}$  & $0.177^{\pm.001}$ \\
\bottomrule
\end{tabular}
}
\vspace{-3mm}
\caption{\label{tab:posegen} \textbf{Pose generation}. We sample $20 \times 500$ poses. $\pm$ indicates the 95$\%$ confidence interval in sampling  $20 \times 500$ poses.\vspace{-1mm}}
\end{table} 

\subsection{Optimization-based Downstream Tasks}\vspace{-1mm}
\newmodel{} can be used as a pose prior term in optimization-based downstream tasks~(\secref{exp:iksolver}) such as pose completion from partial observation or IK solver and pose estimation from images. For each task, we aim to find optimal SMPL parameters ($\pose, \shape$) that explain the observation. The optimization objective is formulated as~\eqnref{eq:opt_generic}, where $\mathcal{L}_{\mathrm{data}}$ is the task-dependent data term, $\mathcal{L}_{\pose}$ is the pose prior term, $\mathcal{L}_{\shape}$ is the shape prior term and $\mathcal{L}_{\alpha}$ represents any other regularizer, if needed. In our experiments we use $\mathcal{L}_{\shape} = || \shape ||^2$~\cite{SMPL-X:2019}. For pose prior terms, VPoser uses $\mathcal{L}_{\theta} = ||z||_2^2$, where $z\,$ is the latent vector of the VAE. 
 Pose-NDF~\cite{tiwari22posendf}, FM-Dis and \newmodel{} use $\mathcal{L}_{\theta} =f(\pose)$, where $f(\pose)$ is the distance value predicted from network. Details of the data and regularizer terms for each task are provided in respective sections. 
\vspace{-1mm}
\begin{equation}
\label{eq:opt_generic}
\hat{\shape}, \hat{\pose}  =  \argmin_{\shape, \pose} \; \;  \mathcal{L}_{\mathrm{data}} + \lambda_{\pose}\mathcal{L}_{\pose} + \lambda_{\shape}\mathcal{L}_{\shape} + \lambda_{\alpha}\mathcal{L}_{\alpha} \textrm{,}
\end{equation}
\vspace{-1mm}
\paragraph{IK solver from partial observations}
\label{exp:iksolver}
Pose acquisition from dense sensors is expensive and tedious, while partial/sparse observation is underconstrained. This underscores the need for a fast and user-friendly Inverse Kinematics (IK) Solver for generating diverse and realistic complete poses from partial observations. To address this, we devise an experimental setup for 3D pose completion from partial observations. Our optimization process, based on \eqnref{eq:opt_generic}, incorporates the $\mathcal{L}_{\mathrm{data}}$ term defined by \eqnref{eq:iksovler}, where $\smpl(\cdot)$ maps ($\beta, \pose$) to SMPL mesh vertices, $\mathcal{J}$ maps the vertices to observed markers/joints, and  $\Vec{J}^{\mathrm{obs}} \in \mathbb{R}^{|\mathrm{joints}| \times 3}$ represents partial marker or joint observations.
\vspace{-1mm}
\begin{equation}
\label{eq:iksovler}
\mathcal{L}_{\mathrm{data}} =  \sum_{i \in \mathrm{joints}} ||( \mathcal{J}(\smpl( \beta, \pose))_{j} - \Vec{J}^{\mathrm{obs}}_j) ||_2
\end{equation}

\begin{table*}[th]
	\begin{center}
		\centering
		
 		\resizebox{1.0\textwidth}{!}{
		\begin{tabular}{l|ccc|ccc|ccc}
		    \toprule
		     \multirow{2}{*}{Method} & \multicolumn{3}{c|}{Occ. Single Arm}& \multicolumn{3}{c|}{Only End Effectors Visible} & \multicolumn{3}{c}{Occ. Legs} \\
			 & \fid{} $\downarrow$ & APD (in cm) $\uparrow$\  & $\distnn$ (in rad) $\downarrow$\ & \fid{} $\downarrow$ & APD (in cm) $\uparrow$\   & $\distnn$ (in rad) $\downarrow$\  & \fid{} $\downarrow$ & APD (in cm) $\uparrow$\  & $\distnn$ (in rad) $\downarrow$\  \\
			\hline
            VPoser-Random  &  $1.148^{\pm.264}$ & $3.218^{\pm.553}$ &  $0.069^{\pm.000}$ & $0.769^{\pm.095}$   & $6.706^{\pm.625}$   & $0.068^{\pm.000}$ &  $0.650^{\pm.150}$ & $9.399^{\pm1.368}$  &  $0.060^{\pm.004}$ \\   
			Pose-NDF~\cite{tiwari22posendf} & $1.281^{\pm.258}$  & $15.294^{\pm1.927}$ & $0.443^{\pm.001}$   &  $1.964^{\pm.125}$ &  $30.871^{\pm1.202}$  & $0.643^{\pm.001}$  & $3.043^{\pm.427}$  &  $30.291^{\pm1.987}$  &  $0.548^{\pm.001}$ \\
			FM-Dis & $1.341^{\pm.246}$  &  $4.490^{\pm1.293}$ & $0.154^{\pm.001}$   &  $1.472^{\pm.252}$ &  $9.554^{\pm2.977}$  & $0.153^{\pm.001}$ & $1.030^{\pm.221}$ &  $7.950^{\pm2.773}$  &  $0.155^{\pm.001}$\\
			  \textbf{Ours}  & $1.248^{\pm.341}$  &  $6.094^{\pm.003}$  & $0.137^{\pm.000}$   & $1.006^{\pm.144}$ & $9.787^{\pm.040}$ & $0.143^{\pm.000}$ & $0.887^{\pm.170}$  &  $8.264^{\pm.007}$  & $0.130^{\pm.000}$ \\   
			\bottomrule
            \end{tabular}
            }\vspace{-3mm}
        \caption{\textbf{Quantitative results for IK Solver from with partial/sparse markers.} We run all evaluations 20 times, $\pm$ indicates the 95$\%$ confidence interval. We evaluate under 3 settings: \textbf{Occ. Single Arm},   \textbf{Only End Effectors}  (wrists and ankles) \textbf{Visible} and \textbf{Occ. Legs}. Our method generates more diverse poses than VPoser~\cite{SMPL-X:2019} for invisible body parts, while preserving more realistic poses (smaller distance to the manifold) than Pose-NDF~\cite{tiwari22posendf} and FM-Dis.\vspace{-5mm}}
		\label{tab:iksolver}
	\end{center}
\end{table*}

We evaluate IK solver on three kinds of observations: 1) \textbf{Occluded Single Arm}, 2) \textbf{Only End Effectors Visible}, and 3) \textbf{Occluded Legs}. In~\figref{fig:partialIK} (b), we present qualitative results from our experiments, where multiple hypotheses are generated based on different initializations given a partial observation. Note that VPoser~\cite{SMPL-X:2019}'s default setting initializes the latent space with a zero vector without introducing noise. Thus we additionally fine-tune it with random initialization of latent space, denoting as \textbf{VPoser-Random}, where the initialization is sampled from a standard normal distribution $\mathcal{N}(0, \mathbf{I})$. Our findings show that VPoser-Random produces less diverse poses than distance field-based approaches due to the Gaussian assumption in VPoser's latent space. Additionally, Pose-NDF often generates unrealistic poses and is highly dependent on the initialization. In contrast, \newmodel{} generates realistic and diverse poses. 

Quantitative analysis in~\tabref{tab:iksolver} includes evaluating pose diversity and realism using \fid{}, APD, and $\distnn$ metrics. Notably, VPoser tends to generate more common poses with slightly better \fid{} and $\distnn$ scores due to its generation of almost mean poses, consistently appearing realistic. Pose-NDF exhibits the highest APD, \fid{}, and $\distnn$ scores, indicating the generation of diverse yet unrealistic poses. These observations are consistent with the results shown in~\figref{fig:partialIK} (b). Our baseline, FM-Dis, performs poorly in terms of diversity and realism. Additional results on more experimental setups are provided in the supplementals.

\paragraph{Monocular 3D pose estimation from images}
\label{exp:hpsimages}
3D pose estimation with neural networks is characterized by speed and robustness; however, it tends to lack accuracy due to the absence of feedback between the observation and prediction. To improve predictions, we propose refining the predictions through an optimization pipeline based on~\eqnref{eq:opt_generic}. We use the \sota{} pose-estimation model, SMPLer-X~\cite{cai2023smpler} for predictions and evaluate the optimization pipeline on the 3DPW~\cite{vonMarcard2018} dataset. The data term in~\eqnref{eq:opt_generic} is: 
\vspace{-1mm}
\begin{equation}
\label{eq:image}
\mathcal{L}_{\mathrm{data}} = \sum_{i \in \mathrm{joints}} \gamma_i w_i \rho(\Pi_K (R_{\theta}(J(\shape))) - \hat{J}_i)
\vspace{-1pt}
\end{equation}
where $\hat{J}_i$ are GT(or predicted) 2D-keypoints, $R_{\theta}$ transforms the joints along the kinematic tree according to the pose $\pose$, $\Pi_K$ is 3D-2D projection with intrinsic camera parameters, $\rho$ is a robust Geman-McClure error~\cite{geman1986bayesian}, $w_i$ are conf. factor of 2d keypoint prediction and $\gamma_i$ is joint weight. 

In~\tabref{tab:image}, we compare \newmodel{}-based optimization with other pose prior such as VPoser, Pose-NDF, FM-Dis, and also with NoPrior term. For quantitative evaluation, we measure PA-MPJPE, PA-PVE, and PCK@50mm~\cite{cai2023smpler,vonMarcard2018} on 3DPW dataset. Our experiments demonstrate that the optimization method based on \newmodel{} consistently outperforms other pose priors. This highlights that the \newmodel{}-based pose manifold is more detailed, leading to improved accuracy while preserving the realism of poses. We show qualitative results of SMPLer-X prediction and refined results using \newmodel{} based optimization in~\figref{fig:image_more_results} and provide more results in the supplementary material.

\begin{figure}[t]
\vspace{-3mm}
\centering
    \begin{overpic}[trim=0cm 0cm 0cm 0cm,clip, width=0.98\linewidth]{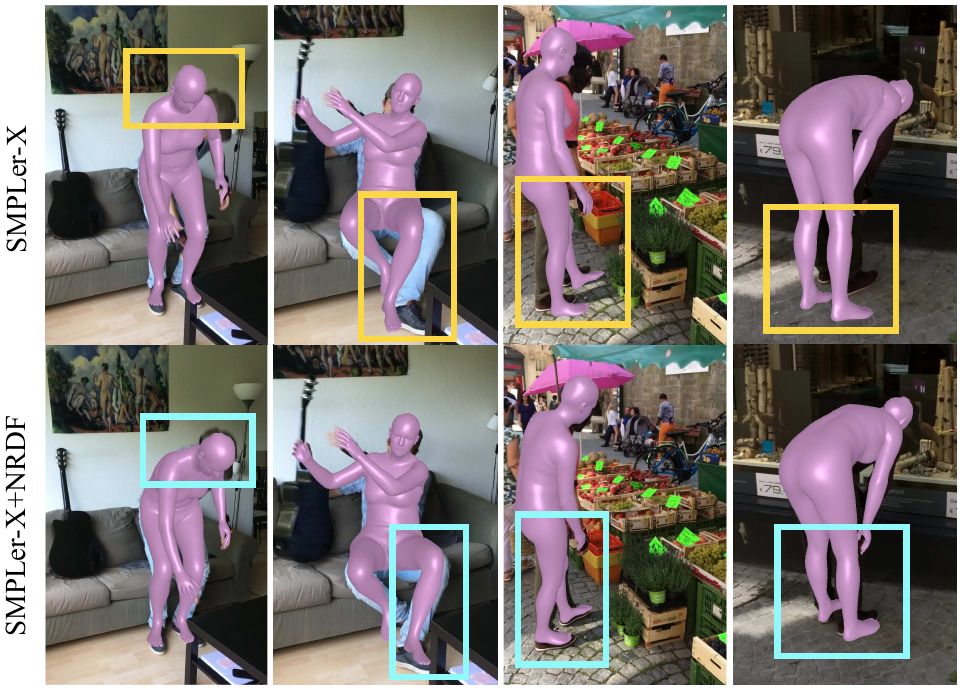}
    \end{overpic}
    \vspace{-3mm}
	\caption{ \textbf{3D pose and shape estimation from images:}(Top): Results from SMPLer-X~\cite{cai2023smpler}, (Bottom): We refine the network prediction using \newmodel{} based optimization pipeline. As highlighted, refined poses align better with the observation.\vspace{-2mm}}
	\label{fig:image_more_results}
 \vspace{-0.5pt}
\end{figure}

\begin{table}[t]
\centering

\resizebox{\columnwidth}{!}{
\begin{tabular}{lccccccc}
\toprule
Method &  PA-MPJPE$\downarrow$\ (in mm) & PA-PVE$\downarrow$\ (in mm)  & PCK@50$\uparrow$ \\

\midrule
SMPLer-X~\cite{cai2023smpler} & 77.82 &  62.77  & 59.30 \\
+ No prior  &  77.97 &  63.47  & 60.25 \\
+ VPoser~\cite{SMPL-X:2019}  &   73.44 & 59.87  & 60.49 \\
+ Pose-NDF~\cite{tiwari22posendf}  &  70.19   & 55.30  & 62.06 \\
+ FM-Dis  &  71.53 & 56.82   &  61.33 \\
+ Ours  &   \textbf{68.47} & \textbf{53.29}  & \textbf{66.17}\\
\bottomrule
\end{tabular}
}
\vspace{-3mm}
\caption{\label{tab:image} \textbf{3D pose and shape estimation from images:} We take SMPLer-X~\cite{cai2023smpler} predictions and refine them using optimization pipeline. We compare the performance of different pose priors.\vspace{-1mm}}
\end{table}

\vspace{-1mm}
\subsection{Extending \newmodel \ to Other Articulated Bodies}
\label{exp:animalhand}
\vspace{-1mm}
As \newmodel{} formulation is not limited to human poses, we use the same formulation to model manifolds of plausible hand and animal poses. 
For hands, we use MANO representation~\cite{MANO:SIGGRAPHASIA:2017} and the DART dataset~\cite{gao2022dart}, covering 80K poses of the right hand, to learn a right-hand pose prior.
For animals, we use SMALR~\cite{Zuffi_2017_CVPR, zuffi2018lions} representation and utilize the Animal3D dataset~\cite{Xu_2023_ICCV}, covering 3K animal poses, to learn articulated animal pose priors. Specifically, we train two different priors for dogs and horses. We show diverse pose generation results in and~\figref{fig:teaser} (right) and supplementals.
\vspace{-0.5cm}
\section{Conclusion and Discussion}
\vspace{-0.2cm}
We presented Neural Riemannian Distance Fields (\newmodel{}), which are data-driven priors that model the space of plausible articulations. The pose prior is represented as the zero-level-set of a neural field in a high-dimensional product-quaternion space. Our model is trained to predict distance geodesic distance on the Riemannian pose manifold. We introduce crucial technical insights to effectively learn a well-behaved and detailed pose manifold. 1) We introduce a sampling framework on Riemannian manifold, that follows the desired distribution, 2) A Riemannian distance metric and 3) We develop a theoretically sound adaptive-step Riemannian gradient descent algorithm that accelerates the convergence in mapping poses onto the learned manifold. Furthermore, we establish connections with Riemannian flow matching~\cite{lipman2022flow} and introduce baselines based on RFM to demonstrate the advantages of \newmodel{}. Our model demonstrates effectiveness in various applications, including pose generation, optimization-based Inverse Kinematics (IK) solving, and 3D pose estimation from images.  We also show the versatility of our formulation by extending it to learning pose priors for hands and animals.

\paragraph{Limitations and future work} Since our approach is based on an iterative sampling scheme, 
it may slightly reduce efficiency for pose generation compared to directly mapping a random latent code to a pose. 
Rather than merely sampling the initial point, we could also \emph{inject noise} during the projection, transforming it into a sequential geodesic MCMC sampler. This would be effective at generating a variety of random poses similar to a given initial pose. 
We could also model uncertainty over the manifold by describing it as a distribution over a family of implicit surfaces.
We leave these promising avenues for future research. 

{\small
\paragraph{Acknowledgments:} 
We thank RVH members, and reviewers for their feedback.
The project was made possible by funding from the Carl Zeiss Foundation. 
This work is supported by the Deutsche Forschungsgemeinschaft (DFG, German Research Foundation) - 409792180 (Emmy Noether Programme, project:  Real Virtual Humans) and the German Federal Ministry of Education and Research (BMBF): Tübingen AI Center, FKZ: 01IS18039A. 
Gerard Pons-Moll is a member of the Machine Learning Cluster of Excellence, EXC number 2064/1 - Project number 390727645. This work was supported by the Engineering and Physical Sciences Research Council [grant number EP/X011364/1].
}

{
    \small
    \bibliographystyle{ieeenat_fullname}
    \bibliography{main}

\begin{thebibliography}{73}
\providecommand{\natexlab}[1]{#1}
\providecommand{\url}[1]{\texttt{#1}}
\expandafter\ifx\csname urlstyle\endcsname\relax
  \providecommand{\doi}[1]{doi: #1}\else
  \providecommand{\doi}{doi: \begingroup \urlstyle{rm}\Url}\fi

\bibitem[Absil and Gallivan(2009)]{absil2009accelerated}
P-A Absil and Kyle~A Gallivan.
\newblock Accelerated line-search and trust-region methods.
\newblock \emph{SIAM Journal on Numerical Analysis}, 47\penalty0 (2):\penalty0 997--1018, 2009.

\bibitem[Akhter and Black(2015)]{akhter_jointlimits}
Ijaz Akhter and Michael~J. Black.
\newblock Pose-conditioned joint angle limits for {3D} human pose reconstruction.
\newblock In \emph{CVPR}, 2015.

\bibitem[Aliakbarian et~al.(2020)Aliakbarian, Sadat~Saleh, Salzmann, Petersson, and Gould]{stoch_diverse}
Sadegh Aliakbarian, Fatemeh Sadat~Saleh, Mathieu Salzmann, Lars Petersson, and Stephen Gould.
\newblock A stochastic conditioning scheme for diverse human motion prediction.
\newblock In \emph{CVPR}, 2020.

\bibitem[Alldieck et~al.(2019)Alldieck, Magnor, Bhatnagar, Theobalt, and Pons-Moll]{alldieck19cvpr}
Thiemo Alldieck, Marcus Magnor, Bharat~Lal Bhatnagar, Christian Theobalt, and Gerard Pons-Moll.
\newblock Learning to reconstruct people in clothing from a single {RGB} camera.
\newblock In \emph{CVPR}, 2019.

\bibitem[Angulo(2014)]{angulo2014riemannian}
Jesus Angulo.
\newblock Riemannian l p averaging on lie group of nonzero quaternions.
\newblock \emph{Advances in Applied Clifford Algebras}, 24\penalty0 (2):\penalty0 355--382, 2014.

\bibitem[Baerlocher and Boulic(2000)]{range_ballsocket_twist}
Paolo Baerlocher and Ronan Boulic.
\newblock Parametrization and range of motion of the ball-and-socket joint.
\newblock In \emph{Proceedings of the IFIP TC5/WG5.10 DEFORM'2000 Workshop and AVATARS'2000 Workshop on Deformable Avatars}, 2000.

\bibitem[Barsoum et~al.(2018)Barsoum, Kender, and Liu]{BarsoumCVPRW2018}
Emad Barsoum, John Kender, and Zicheng Liu.
\newblock {HP-GAN:} probabilistic {3D} human motion prediction via gan.
\newblock In \emph{CVPR Workshops}, 2018.

\bibitem[Bhatnagar et~al.(2019)Bhatnagar, Tiwari, Theobalt, and Pons-Moll]{bhatnagar2019mgn}
Bharat~Lal Bhatnagar, Garvita Tiwari, Christian Theobalt, and Gerard Pons-Moll.
\newblock Multi-garment net: Learning to dress {3D} people from images.
\newblock In \emph{ICCV}, 2019.

\bibitem[Bhatnagar et~al.(2022)Bhatnagar, Xie, Petrov, Sminchisescu, Theobalt, and Pons-Moll]{bhatnagar22behave}
Bharat~Lal Bhatnagar, Xianghui Xie, Ilya Petrov, Cristian Sminchisescu, Christian Theobalt, and Gerard Pons-Moll.
\newblock {BEHAVE}: Dataset and method for tracking human object interactions.
\newblock In \emph{CVPR}, 2022.

\bibitem[Birdal and Simsekli(2019)]{birdal2019probabilistic}
Tolga Birdal and Umut Simsekli.
\newblock Probabilistic permutation synchronization using the riemannian structure of the birkhoff polytope.
\newblock In \emph{Proceedings of the IEEE/CVF Conference on Computer Vision and Pattern Recognition}, pages 11105--11116, 2019.

\bibitem[Birdal et~al.(2018)Birdal, Simsekli, Eken, and Ilic]{birdal2018bayesian}
Tolga Birdal, Umut Simsekli, Mustafa~Onur Eken, and Slobodan Ilic.
\newblock Bayesian pose graph optimization via bingham distributions and tempered geodesic mcmc.
\newblock \emph{Advances in Neural Information Processing Systems}, 31, 2018.

\bibitem[Birdal et~al.(2020)Birdal, Arbel, Simsekli, and Guibas]{birdal2020synchronizing}
Tolga Birdal, Michael Arbel, Umut Simsekli, and Leonidas~J Guibas.
\newblock Synchronizing probability measures on rotations via optimal transport.
\newblock In \emph{Proceedings of the IEEE/CVF Conference on Computer Vision and Pattern Recognition}, pages 1569--1579, 2020.

\bibitem[Bogo et~al.(2016)Bogo, Kanazawa, Lassner, Gehler, Romero, and Black]{bogo2016keep}
Federica Bogo, Angjoo Kanazawa, Christoph Lassner, Peter Gehler, Javier Romero, and Michael~J Black.
\newblock Keep it {SMPL}: Automatic estimation of {3D} human pose and shape from a single image.
\newblock In \emph{ECCV}, 2016.

\bibitem[Bogo et~al.(2017)Bogo, Romero, Pons-Moll, and Black]{dfaust}
Federica Bogo, Javier Romero, Gerard Pons-Moll, and Michael~J. Black.
\newblock Dynamic {FAUST}: {R}egistering human bodies in motion.
\newblock In \emph{IEEE Conf. on Computer Vision and Pattern Recognition (CVPR)}, 2017.

\bibitem[Bonnabel(2013)]{bonnabel2013stochastic}
Silvere Bonnabel.
\newblock Stochastic gradient descent on riemannian manifolds.
\newblock \emph{IEEE Transactions on Automatic Control}, 58\penalty0 (9):\penalty0 2217--2229, 2013.

\bibitem[Boumal(2020)]{boumal2020introduction}
Nicolas Boumal.
\newblock An introduction to optimization on smooth manifolds.
\newblock \emph{Available online, May}, 2020.

\bibitem[Cai et~al.(2020)Cai, Yang, Averbuch-Elor, Hao, Belongie, Snavely, and Hariharan]{ShapeGF}
Ruojin Cai, Guandao Yang, Hadar Averbuch-Elor, Zekun Hao, Serge Belongie, Noah Snavely, and Bharath Hariharan.
\newblock Learning gradient fields for shape generation.
\newblock In \emph{Proceedings of the European Conference on Computer Vision (ECCV)}, 2020.

\bibitem[Cai et~al.(2023)Cai, Yin, Zeng, Wei, Sun, Wang, Pang, Mei, Zhang, Zhang, et~al.]{cai2023smpler}
Zhongang Cai, Wanqi Yin, Ailing Zeng, Chen Wei, Qingping Sun, Yanjun Wang, Hui~En Pang, Haiyi Mei, Mingyuan Zhang, Lei Zhang, et~al.
\newblock Smpler-x: Scaling up expressive human pose and shape estimation.
\newblock \emph{arXiv preprint arXiv:2309.17448}, 2023.

\bibitem[Chen et~al.(2010)Chen, Zhuang, Nie, Yang, Wu, and Xiao]{posemetric_GPD}
Cheng Chen, Yueting Zhuang, Feiping Nie, Yi Yang, Fei Wu, and Jun Xiao.
\newblock Learning a 3d human pose distance metric from geometric pose descriptor.
\newblock \emph{IEEE transactions on visualization and computer graphics}, 17, 2010.

\bibitem[Chen et~al.(2022)Chen, Yin, Birdal, Chen, Guibas, and Wang]{chen2022projective}
Jiayi Chen, Yingda Yin, Tolga Birdal, Baoquan Chen, Leonidas~J Guibas, and He Wang.
\newblock Projective manifold gradient layer for deep rotation regression.
\newblock In \emph{Proceedings of the IEEE/CVF Conference on Computer Vision and Pattern Recognition}, pages 6646--6655, 2022.

\bibitem[Chen and Lipman(2023)]{chen2023riemannian}
Ricky~TQ Chen and Yaron Lipman.
\newblock Riemannian flow matching on general geometries.
\newblock \emph{arXiv preprint arXiv:2302.03660}, 2023.

\bibitem[Chen et~al.(2018)Chen, Rubanova, Bettencourt, and Duvenaud]{chen2018neural}
Ricky~TQ Chen, Yulia Rubanova, Jesse Bettencourt, and David~K Duvenaud.
\newblock Neural ordinary differential equations.
\newblock \emph{Advances in neural information processing systems}, 31, 2018.

\bibitem[Chibane et~al.(2020{\natexlab{a}})Chibane, Alldieck, and Pons-Moll]{chibane20ifnet}
Julian Chibane, Thiemo Alldieck, and Gerard Pons-Moll.
\newblock Implicit functions in feature space for {3D} shape reconstruction and completion.
\newblock In \emph{{IEEE} Conference on Computer Vision and Pattern Recognition (CVPR)}. {IEEE}, 2020{\natexlab{a}}.

\bibitem[Chibane et~al.(2020{\natexlab{b}})Chibane, Mir, and Pons-Moll]{chibane2020ndf}
Julian Chibane, Aymen Mir, and Gerard Pons-Moll.
\newblock Neural unsigned distance fields for implicit function learning.
\newblock In \emph{Advances in Neural Information Processing Systems ({NeurIPS})}, 2020{\natexlab{b}}.

\bibitem[Ci et~al.(2022)Ci, Wu, Zhu, Ma, Dong, Zhong, and Wang]{ci2022gfpose}
Hai Ci, Mingdong Wu, Wentao Zhu, Xiaoxuan Ma, Hao Dong, Fangwei Zhong, and Yizhou Wang.
\newblock Gfpose: Learning 3d human pose prior with gradient fields.
\newblock \emph{arXiv preprint arXiv:2212.08641}, 2022.

\bibitem[Dabral et~al.(2023)Dabral, Mughal, Golyanik, and Theobalt]{dabral2022mofusion}
Rishabh Dabral, Muhammad~Hamza Mughal, Vladislav Golyanik, and Christian Theobalt.
\newblock Mofusion: A framework for denoising-diffusion-based motion synthesis.
\newblock In \emph{Computer Vision and Pattern Recognition (CVPR)}, 2023.

\bibitem[Davydov et~al.(2022)Davydov, Remizova, Constantin, Honari, Salzmann, and Fua]{Davydov_2022_CVPR}
Andrey Davydov, Anastasia Remizova, Victor Constantin, Sina Honari, Mathieu Salzmann, and Pascal Fua.
\newblock Adversarial parametric pose prior.
\newblock In \emph{CVPR}, 2022.

\bibitem[Do~Carmo and Flaherty~Francis(1992)]{do1992riemannian}
Manfredo~Perdigao Do~Carmo and J Flaherty~Francis.
\newblock \emph{Riemannian geometry}.
\newblock Springer, 1992.

\bibitem[Engell-Nørregård et~al.(2012)Engell-Nørregård, Niebe, and Erleben]{euler_jointlimit}
Morten Engell-Nørregård, Sarah Niebe, and Kenny Erleben.
\newblock A joint-constraint model for human joints using signed distance-fields.
\newblock \emph{Multibody System Dynamics}, 28, 2012.

\bibitem[Falorsi et~al.(2019)Falorsi, de~Haan, Davidson, and Forr{\'e}]{falorsi2019reparameterizing}
Luca Falorsi, Pim de Haan, Tim~R Davidson, and Patrick Forr{\'e}.
\newblock Reparameterizing distributions on lie groups.
\newblock In \emph{The 22nd International Conference on Artificial Intelligence and Statistics}, pages 3244--3253. PMLR, 2019.

\bibitem[Gao et~al.(2022)Gao, Xiu, Li, Yang, Wang, Zhang, Zhang, Lu, and Tan]{gao2022dart}
Daiheng Gao, Yuliang Xiu, Kailin Li, Lixin Yang, Feng Wang, Peng Zhang, Bang Zhang, Cewu Lu, and Ping Tan.
\newblock {DART: Articulated Hand Model with Diverse Accessories and Rich Textures}.
\newblock In \emph{Thirty-sixth Conference on Neural Information Processing Systems Datasets and Benchmarks Track}, 2022.

\bibitem[Geman and Geman(1986)]{geman1986bayesian}
Donald Geman and Stuart Geman.
\newblock Bayesian image analysis.
\newblock In \emph{Disordered systems and biological organization}, pages 301--319. Springer, 1986.

\bibitem[Guzov et~al.(2021)Guzov, Mir, Sattler, and Pons-Moll]{HPS}
Vladimir Guzov, Aymen Mir, Torsten Sattler, and Gerard Pons-Moll.
\newblock {Human POSEitioning System (HPS)}: {3D} human pose estimation and self-localization in large scenes from body-mounted sensors.
\newblock In \emph{CVPR}, 2021.

\bibitem[Hirschorn and Avidan(2023)]{hirschorn2023normalizing}
Or Hirschorn and Shai Avidan.
\newblock Normalizing flows for human pose anomaly detection, 2023.

\bibitem[Ho et~al.(2020)Ho, Jain, and Abbeel]{ho2020denoising}
Jonathan Ho, Ajay Jain, and Pieter Abbeel.
\newblock Denoising diffusion probabilistic models, 2020.

\bibitem[Ionescu et~al.(2014)Ionescu, Papava, Olaru, and Sminchisescu]{h36m_pami}
Catalin Ionescu, Dragos Papava, Vlad Olaru, and Cristian Sminchisescu.
\newblock Human3.6m: Large scale datasets and predictive methods for 3d human sensing in natural environments.
\newblock \emph{IEEE Transactions on Pattern Analysis and Machine Intelligence}, 36\penalty0 (7):\penalty0 1325--1339, 2014.

\bibitem[Johnson et~al.(2019)Johnson, Douze, and J{\'e}gou]{johnson2019billion}
Jeff Johnson, Matthijs Douze, and Herv{\'e} J{\'e}gou.
\newblock Billion-scale similarity search with {GPUs}.
\newblock \emph{IEEE Transactions on Big Data}, 7\penalty0 (3):\penalty0 535--547, 2019.

\bibitem[Kanazawa et~al.(2018)Kanazawa, Black, Jacobs, and Malik]{hmrKanazawa17}
Angjoo Kanazawa, Michael~J. Black, David~W. Jacobs, and Jitendra Malik.
\newblock End-to-end recovery of human shape and pose.
\newblock In \emph{CVPR}, 2018.

\bibitem[Kobyzev et~al.(2021)Kobyzev, Prince, and Brubaker]{Kobyzev_2021}
Ivan Kobyzev, Simon~J.D. Prince, and Marcus~A. Brubaker.
\newblock Normalizing flows: An introduction and review of current methods.
\newblock \emph{IEEE Transactions on Pattern Analysis and Machine Intelligence}, 43\penalty0 (11):\penalty0 3964–3979, 2021.

\bibitem[Kocabas et~al.(2020)Kocabas, Athanasiou, and Black]{vibe}
Muhammed Kocabas, Nikos Athanasiou, and Michael~J. Black.
\newblock Vibe: Video inference for human body pose and shape estimation.
\newblock In \emph{CVPR}, 2020.

\bibitem[Kochurov et~al.(2020)Kochurov, Karimov, and Kozlukov]{geoopt2020kochurov}
Max Kochurov, Rasul Karimov, and Serge Kozlukov.
\newblock Geoopt: Riemannian optimization in pytorch, 2020.

\bibitem[Kolotouros et~al.(2019)Kolotouros, Pavlakos, Black, and Daniilidis]{SPIN:ICCV:2019}
Nikos Kolotouros, Georgios Pavlakos, Michael~J. Black, and Kostas Daniilidis.
\newblock Learning to reconstruct {3D} human pose and shape via model-fitting in the loop.
\newblock In \emph{ICCV}, 2019.

\bibitem[Krueger et~al.(2010)Krueger, Tautges, Weber, and Zinke]{posemetric_localfast}
Bjoern Krueger, Jochen Tautges, Andreas Weber, and Arno Zinke.
\newblock {Fast Local and Global Similarity Searches in Large Motion Capture Databases}.
\newblock In \emph{Eurographics/ ACM SIGGRAPH Symposium on Computer Animation}. The Eurographics Association, 2010.

\bibitem[Lipman et~al.(2022)Lipman, Chen, Ben-Hamu, Nickel, and Le]{lipman2022flow}
Yaron Lipman, Ricky~TQ Chen, Heli Ben-Hamu, Maximilian Nickel, and Matthew Le.
\newblock Flow matching for generative modeling.
\newblock In \emph{The Eleventh International Conference on Learning Representations}, 2022.

\bibitem[Mahmood et~al.(2019)Mahmood, Ghorbani, Troje, Pons-Moll, and Black]{AMASS:2019}
Naureen Mahmood, Nima Ghorbani, Nikolaus~F. Troje, Gerard Pons-Moll, and Michael~J. Black.
\newblock {AMASS}: Archive of motion capture as surface shapes.
\newblock In \emph{ICCV}, 2019.

\bibitem[Mattila(1999)]{mattila1999geometry}
Pertti Mattila.
\newblock \emph{Geometry of sets and measures in Euclidean spaces: fractals and rectifiability}.
\newblock Number~44. Cambridge university press, 1999.

\bibitem[Munier(2007)]{munier2007steepest}
Julien Munier.
\newblock Steepest descent method on a riemannian manifold: the convex case.
\newblock \emph{Balkan Journal of Geometry \& Its Applications}, 12\penalty0 (2), 2007.

\bibitem[Omran et~al.(2018)Omran, Lassner, Pons-Moll, Gehler, and Schiele]{omran2018neural}
Mohamed Omran, Christop Lassner, Gerard Pons-Moll, Peter Gehler, and Bernt Schiele.
\newblock Neural body fitting: Unifying deep learning and model based human pose and shape estimation.
\newblock In \emph{3DV}, 2018.

\bibitem[Ormoneit et~al.(2000)Ormoneit, Sidenbladh, Black, and Hastie]{pca_cyclic}
Dirk Ormoneit, Hedvig Sidenbladh, Michael Black, and Trevor Hastie.
\newblock Learning and tracking cyclic human motion.
\newblock \emph{Advances in Neural Information Processing Systems}, 13, 2000.

\bibitem[Park et~al.(2019)Park, Florence, Straub, Newcombe, and Lovegrove]{Park_2019_CVPR}
Jeong~Joon Park, Peter Florence, Julian Straub, Richard Newcombe, and Steven Lovegrove.
\newblock {DeepSDF}: Learning continuous signed distance functions for shape representation.
\newblock In \emph{CVPR}, 2019.

\bibitem[Pavlakos et~al.(2019)Pavlakos, Choutas, Ghorbani, Bolkart, Osman, Tzionas, and Black]{SMPL-X:2019}
Georgios Pavlakos, Vasileios Choutas, Nima Ghorbani, Timo Bolkart, Ahmed A.~A. Osman, Dimitrios Tzionas, and Michael~J. Black.
\newblock Expressive body capture: {3D} hands, face, and body from a single image.
\newblock In \emph{CVPR}, 2019.

\bibitem[Petrovich et~al.(2021)Petrovich, Black, and Varol]{ACTOR:ICCV:2021}
Mathis Petrovich, Michael~J. Black, and G\"{u}l Varol.
\newblock Action-conditioned {3D} human motion synthesis with transformer {VAE}.
\newblock In \emph{ICCV}, 2021.

\bibitem[Rempe et~al.(2021)Rempe, Birdal, Hertzmann, Yang, Sridhar, and Guibas]{rempe2021humor}
Davis Rempe, Tolga Birdal, Aaron Hertzmann, Jimei Yang, Srinath Sridhar, and Leonidas~J. Guibas.
\newblock {HuMoR}: {3D} human motion model for robust pose estimation.
\newblock In \emph{ICCV}, 2021.

\bibitem[Romero et~al.(2017)Romero, Tzionas, and Black]{MANO:SIGGRAPHASIA:2017}
Javier Romero, Dimitrios Tzionas, and Michael~J. Black.
\newblock Embodied hands: Modeling and capturing hands and bodies together.
\newblock \emph{ACM Transactions on Graphics, (Proc. SIGGRAPH Asia)}, 36\penalty0 (6), 2017.

\bibitem[Salimans et~al.(2016)Salimans, Goodfellow, Zaremba, Cheung, Radford, Chen, and Chen]{NIPS2016_8a3363ab}
Tim Salimans, Ian Goodfellow, Wojciech Zaremba, Vicki Cheung, Alec Radford, Xi Chen, and Xi Chen.
\newblock Improved techniques for training gans.
\newblock In \emph{Advances in Neural Information Processing Systems}. Curran Associates, Inc., 2016.

\bibitem[Shao and Ng-Thow-Hing(2003)]{shao_realisitic}
Wei Shao and Victor Ng-Thow-Hing.
\newblock A general joint component framework for realistic articulation in human characters.
\newblock In \emph{Proceedings of the 2003 symposium on Interactive {3D} graphics}, pages 11--18, 2003.

\bibitem[Sidenbladh et~al.(2000)Sidenbladh, Black, , and Fleet]{Black:ECCV:2000}
H. Sidenbladh, M.~J. Black, , and D.J. Fleet.
\newblock Stochastic tracking of {3D} human figures using {2D} image motion.
\newblock In \emph{ECCV}, 2000.

\bibitem[Song and Ermon(2019)]{song_gradient}
Yang Song and Stefano Ermon.
\newblock Generative modeling by estimating gradients of the data distribution.
\newblock In \emph{Proceedings of the 33rd International Conference on Neural Information Processing Systems}, Red Hook, NY, USA, 2019. Curran Associates Inc.

\bibitem[Song et~al.(2021)Song, Sohl-Dickstein, Kingma, Kumar, Ermon, and Poole]{song2021scorebased}
Yang Song, Jascha Sohl-Dickstein, Diederik~P Kingma, Abhishek Kumar, Stefano Ermon, and Ben Poole.
\newblock Score-based generative modeling through stochastic differential equations.
\newblock In \emph{International Conference on Learning Representations}, 2021.

\bibitem[Stoll et~al.(2010)Stoll, Gall, de~Aguiar, Thrun, and Theobalt]{stoll_videobased}
Carsten Stoll, Juergen Gall, Edilson de Aguiar, Sebastian Thrun, and Christian Theobalt.
\newblock Video-based reconstruction of animatable human characters.
\newblock In \emph{ACM SIGGRAPH Asia}, 2010.

\bibitem[Tevet et~al.(2023)Tevet, Raab, Gordon, Shafir, Cohen-or, and Bermano]{tevet2023human}
Guy Tevet, Sigal Raab, Brian Gordon, Yoni Shafir, Daniel Cohen-or, and Amit~Haim Bermano.
\newblock Human motion diffusion model.
\newblock In \emph{The Eleventh International Conference on Learning Representations}, 2023.

\bibitem[Tiwari et~al.(2021)Tiwari, Sarafianos, Tung, and Pons-Moll]{tiwari21neuralgif}
Garvita Tiwari, Nikolaos Sarafianos, Tony Tung, and Gerard Pons-Moll.
\newblock {Neural-GIF}: Neural generalized implicit functions for animating people in clothing.
\newblock In \emph{ICCV}, 2021.

\bibitem[Tiwari et~al.(2022)Tiwari, Antic, Lenssen, Sarafianos, Tung, and Pons-Moll]{tiwari22posendf}
Garvita Tiwari, Dimitrije Antic, Jan~Eric Lenssen, Nikolaos Sarafianos, Tony Tung, and Gerard Pons-Moll.
\newblock Pose-ndf: Modeling human pose manifolds with neural distance fields.
\newblock In \emph{European Conference on Computer Vision ({ECCV})}. {Springer}, 2022.

\bibitem[Townsend et~al.(2016)Townsend, Koep, and Weichwald]{pymanopt}
J. Townsend, N. Koep, and S. Weichwald.
\newblock {P}y{M}anopt: a {P}ython toolbox for optimization on manifolds using automatic differentiation.
\newblock \emph{Journal of Machine Learning Research}, 17\penalty0 (137):\penalty0 1--5, 2016.

\bibitem[Tripuraneni et~al.(2018)Tripuraneni, Flammarion, Bach, and Jordan]{tripuraneni2018averaging}
Nilesh Tripuraneni, Nicolas Flammarion, Francis Bach, and Michael~I Jordan.
\newblock Averaging stochastic gradient descent on riemannian manifolds.
\newblock In \emph{Conference On Learning Theory}, pages 650--687. PMLR, 2018.

\bibitem[Urtasun et~al.(2006)Urtasun, Fleet, and Fua]{3dpeople_gaussian}
Raquel Urtasun, David Fleet, and Pascal Fua.
\newblock {3D} people tracking with {G}aussian process dynamical models.
\newblock In \emph{CVPR}, 2006.

\bibitem[von Marcard et~al.(2018)von Marcard, Henschel, Black, Rosenhahn, and Pons-Moll]{vonMarcard2018}
Timo von Marcard, Roberto Henschel, Michael Black, Bodo Rosenhahn, and Gerard Pons-Moll.
\newblock Recovering accurate 3d human pose in the wild using imus and a moving camera.
\newblock In \emph{European Conference on Computer Vision (ECCV)}, 2018.

\bibitem[Wehrbein et~al.(2021)Wehrbein, Rudolph, Rosenhahn, and Wandt]{WehRud2021}
Tom Wehrbein, Marco Rudolph, Bodo Rosenhahn, and Bastian Wandt.
\newblock Probabilistic monocular 3d human pose estimation with normalizing flows.
\newblock In \emph{International Conference on Computer Vision (ICCV)}, 2021.

\bibitem[Xu et~al.(2023)Xu, Zhang, Peng, Ma, Jesslen, Ji, Hu, Zhang, Liu, Wang, Ji, Wang, Yuan, Kaushik, Zhang, Liu, Xie, Cui, Yuille, and Kortylewski]{Xu_2023_ICCV}
Jiacong Xu, Yi Zhang, Jiawei Peng, Wufei Ma, Artur Jesslen, Pengliang Ji, Qixin Hu, Jiehua Zhang, Qihao Liu, Jiahao Wang, Wei Ji, Chen Wang, Xiaoding Yuan, Prakhar Kaushik, Guofeng Zhang, Jie Liu, Yushan Xie, Yawen Cui, Alan Yuille, and Adam Kortylewski.
\newblock Animal3d: A comprehensive dataset of 3d animal pose and shape.
\newblock In \emph{Proceedings of the IEEE/CVF International Conference on Computer Vision (ICCV)}, pages 9099--9109, 2023.

\bibitem[Zhang and Sra(2016)]{zhang2016first}
Hongyi Zhang and Suvrit Sra.
\newblock First-order methods for geodesically convex optimization.
\newblock In \emph{Conference on Learning Theory}, pages 1617--1638. PMLR, 2016.

\bibitem[Zhang et~al.(2021)Zhang, Zhang, Bogo, Pollefeys, and Tang]{zhang2021learning}
Siwei Zhang, Haiyan Zhang, Federica Bogo, Marc Pollefeys, and Siyu Tang.
\newblock Learning motion priors for {4D} human body capture in {3D} scenes.
\newblock In \emph{ICCV}, 2021.

\bibitem[Zuffi et~al.(2017)Zuffi, Kanazawa, Jacobs, and Black]{Zuffi_2017_CVPR}
Silvia Zuffi, Angjoo Kanazawa, David~W. Jacobs, and Michael~J. Black.
\newblock 3d menagerie: Modeling the 3d shape and pose of animals.
\newblock In \emph{Proceedings of the IEEE Conference on Computer Vision and Pattern Recognition (CVPR)}, 2017.

\bibitem[Zuffi et~al.(2018)Zuffi, Kanazawa, and Black]{zuffi2018lions}
Silvia Zuffi, Angjoo Kanazawa, and Michael~J Black.
\newblock Lions and tigers and bears: Capturing non-rigid, 3d, articulated shape from images.
\newblock In \emph{Proceedings of the IEEE conference on Computer Vision and Pattern Recognition}, pages 3955--3963, 2018.

\end{thebibliography}
}

\clearpage
\setcounter{page}{1}
\setcounter{prop}{0}
\setcounter{section}{0}
\renewcommand\thesection{\Alph{section}}
\maketitlesupplementary

In the following, we start with proving the propositions in Sec.~\ref{sec:proofs} and discuss the relationship to Riemannian Flow Matching in Sec.~\ref{sec:rfm_connection}. Then, we present our user study about perceptual pose metrics in Sec.~\ref{sec:user_study}, followed by implementation details in Sec.~\ref{sec:more_implementation_details} and additional results in Sec.~\ref{sec:more_results}.

\section{Proofs \& Theoretical Discussions}
\label{sec:proofs}
\subsection{Proof of Prop. 1}
Before proceeding with the proof let us recall the main proposition. 
\begin{prop}[Quaternion-egrad2rgrad]
For the quaternion manifold, the projection and mapping onto the tangent space of the canonical unit quaternion $\eb=\begin{bmatrix} 1 & 0 & 0 & 0 \end{bmatrix}^\top$ (egrad2rgrad in~\cref{eq:gradf}) takes the form:
\begin{align}
    \Pi_{\q}(\vb) &= \Pm\vb - \frac{\eb^\top \Pm\vb}{1+\q^\top \eb} (\q+\eb)\\
    &= \begin{bmatrix}
        0 &  0 & 0 & 0\\
        -q_2/(1+q_1) &  1 & 0 & 0\\
        -q_3/(1+q_1) &  0 & 1 & 0\\
        -q_4/(1+q_1) &  0 & 0 & 1
    \end{bmatrix}\Pm\vb \textnormal{,}
\end{align}
where $\vb\in\R^4$ and $\Pm\vcentcolon=\Pm(\q)=\Id-\q\q^\top$.
\end{prop}
\begin{proof}[Proof]
First, note that the quaternion $\q$ is also the normal vector at point $\q$. Hence, the projection of any ambient vector onto the tangent space of a quaternion can be obtained by the standard projection onto the plane defined by the normal. In other words, the orthogonal complement of the tangent plane is the line in the direction of $\q$. Hence, we can project any ambient vector $\vb\in\Quats$ onto $\TqM$ as:
\begin{align}
    (\vb-{\q\q^\top\vb})=(\Id-\q\q^\top) \vb = \Pm(\q) \vb.
\end{align}
Next, we need to rotate $\Pm(\q) \vb$ to align with the identity tangent space, $\TiM$ or as in the proposition, $\TeM$. To do so, we utilize the \emph{discrete connection} on $\mathcal{S}^3$, which can be obtained by rotation of tangent planes. Since the tangent planes at $\q$ and $\eb$ are defined by $\q$ and $\eb$ themselves, all we need to do is to find the linear map, \ie a rotation, that aligns $\q$ onto $\eb$. This is the typical vector-rotation formula in $4$-space and is given by: 
\begin{equation}
    \vb - \frac{\eb^\top \vb}{1+\q^\top \eb} (\q+\eb)
\end{equation}
Plugging in $\eb=\begin{bmatrix} 1 & 0 & 0 & 0 \end{bmatrix}^\top$ and re-arranging yields the matrix form given in the proposition. Note that this is unique up to rotation around $\eb$ and well connects to the non-uniqueness of parallel transport\footnote{The \emph{hairy ball theorem} states the nonexistence of a global parameterization of any continuously varying basis vector of $\TxS$ on all of $\mathcal{S}$.}. 
\end{proof}

\subsection{Proof of Prop. 2}
We re-state Prop. 2  of the main paper before delving into the proof.
\begin{prop}[Distance preservation]
Let $\mathcal{P}$ be a distribution over domain $[0,1]$, $\pose \in \mathcal{D}$ a data example, $\hat{\pose} \in \Quats^K$ the output of ~\cref{alg:sampledata_jan} with input $(\pose, \mathcal{P})$, and $d= d(\pose, \hat{\pose})$. Then, for the distribution of resulting distances holds $p(d) = \mathcal{P}$.
\end{prop}
\begin{proof}[Proof]
Our proof closely follows the steps in our algorithm where we sample a Gaussian vector of magnitude $h$, project it on the $3$-sphere (where the distribution is uniform) and take steps. In the sequel, we prove each step.
\begin{dfn}
    A vector $\vb\in\R^n$ is said to be {radially symmetric} if $\forall \Ab \in O(n), \vb\stackrel{d}{=}\Ab\vb$.
\end{dfn}
\begin{lemma}
Let $\vb\in\R^n$ be a \emph{radially symmetric} random vector. Then $\vb/\|\vb\|\sim \Unif(\mathcal{S}^{n-1})$.
\end{lemma}
\begin{proof}
Let $\Ab\in O(n)$ denote an orthogonal matrix and $\mathrm{Proj}$ be the projection onto $\mathcal{S}^{n-1}$. Then the following holds:
\begin{equation}
    \mathrm{Proj}(\vb)\vcentcolon=\frac{\vb}{\|\vb\|} \,\stackrel{d}{=} \, \frac{\Ab\vb}{\|\vb\|} = \mathrm{Proj}(\Ab\vb). 
\end{equation}
Thus, the random vector $\vb/\|\vb\|$ takes all its values over $\mathcal{S}^{n-1}$ and is radially symmetric. This is exactly the definition of $\Unif(\mathcal{S}^{n-1})$, the uniform distribution over $\mathcal{S}^{n-1}$. A more rigorous proof involves geometric measure theory, which is beyond the scope of this work. We refer the reader to~\cite{mattila1999geometry}  for further details.
\end{proof}
Showing that $\vb$ after normalization follows a uniform distribution on the unit sphere in $\TposeS$, we move to the second part, where $\|\pose- h\cdot \vb\| = h$ since  $\|\vb\|=1$ and $h\cdot \vb\in\TposeS$ for any $h>0$. In other words, $h$ moves $\vb$ towards/away from the base $\pose$.
\begin{lemma}
    Let $\vb\in\TposeQuats$. Using the ambient space of $\R^4$ and $\|\vb\|=1$, for all $t$:
    \begin{align}
        d\left(\pose, \Exp_{\pose} (h\cdot \vb) \right) = h\|\vb\|,
    \end{align}
in a sufficiently small interval (respecting the cut-locus), where  $d(\cdot,\cdot)$ denotes the geodesic distance. 
\end{lemma}
\begin{proof}
    The proof follows the do Carmo, ``Riemannian Geometry'', Proposition 3.6~\cite{do1992riemannian}.
\end{proof}
Since by this lemma, the exponential map on the sphere preserves distances to the base within $[0,1]$\footnote{Note that, a common mistake is to assume that all distances are preserved. This is not true and only the distances to the base of the Exp/Log map is preserved.}, it follows that $d({\pose}^\prime, \pose) = h$. With a similar argumentation, when $h\sim \mathcal{P}$, $d({\pose}^\prime, \pose)\sim\mathcal{P}$.
\end{proof}

\paragraph{How are the sample points distributed?}
While our algorithm guarantees that the distances attained as a result of sampling preserve the initial choice of distance distribution, the distribution of the samples themselves can undergo distortion, \ie, they can follow a complicated distribution on $\Quats^K$. Naturally, it is of interest to also understand how the samples are actually distributed on $\Quats^K$. While we cannot provide an explicit analytical form, we will provide an intuition into their distributions below, by following the exposition in~\cite{falorsi2019reparameterizing}.

Quaternions look like $SU(2)$, a Lie group isomorphic to $SO(3)$. Their Lie algebra (tangent space) $\mathfrak{su}(2)$ is isomorphic to the Lie algebra $\mathfrak{so(3)}$. Let us denote the Lie algebra of quaternions by $\lieq$. 
Being the direct product of (non-Abelian) Lie groups, $\Quats^K$ is also a Lie group, with a Lie algebra $\lieQ$.
The adjoint representation of $\x\in\lieQ$, $\adj_{\x}(\y)$, is the matrix representation of the map $[\x,\y]$ called the Lie bracket. For quaternions $\adj_{\x}$ is the linear representation of this map.  
In our sampling algorithm (\cref{alg:sampledata_jan}, to explicitly control the distance distribution, we first sample $\vb\sim\normal(0,\Id)$ in the Lie algebra of $\pose$. We then normalize it yielding a uniform distribution on the $3$-sphere and compose it with Gaussian sampled scalars $h$, resulting in a Gaussian distribution radially and uniform distribution in terms of direction: $\vb/\|\vb\|\sim\Unif(\mathcal{S}^3_h) \mathrm{\,and\,} \|\vb\|_2\sim\normal(0,1)$. We call this tangent distribution $r(\vb)$. Next, we use the exponential map to push $r$ wrapping the samples on the manifold resulting in our \emph{wrapped uniform-spherical distribution}:
    \begin{align}
        p(\pose) = \sum\limits_{\substack{\vb\in\lieQ \\ \Exp_{\bmu}(\vb)=\pose}} r(\vb) \left|\J^{-1}\right|,
    \end{align}
    where $\J$ is the Jacobian map, whose determinant is the change of volume:
    \begin{align}
\left|\J^{-1}\right|\vcentcolon=\mathrm{det}\left( 
 \sum_{k=0}^{\infty} \frac{(-1)^k}{(k+1)!}\adj_{\x}^k\right).
    \end{align}
Intuitively, $\J$ linearly relates the tangent spaces and as such can be computed, similar to~\cref{eq:egrad2rgrad}, via the parallel transport. As a result, we lose the simple form of the distribution in the tangent space and can only arrive at the final distribution via this push-forward operation.

\subsection{Riemannian Flow Matching (RFM) vs. NRDF}
\label{sec:rfm_connection}
Our training strategy resembles the extension of a recent state-of-the-art generative model, Flow Matching (FM)~\cite{lipman2022flow} onto Riemannian manifolds, known as, Riemannian Flow Matching (RFM)~\cite{chen2023riemannian}. The differences lie in the sampling of training data and time steps as well as the way the \emph{flow} is computed / predicted. In what follows, we will briefly summarize RFM and make the connection to our model NRDF. We will begin by recalling certain definitions.
\begin{dfn}[Riemannian CNF]
    A CNF $\varphi_t(\cdot):\Man\to\Man$ on the smooth manifold $\Man$ is defined by integration along a time-dependent vector field $v_t(\x)\in\TxM$: $\dot\varphi_t(\x)=v_t(\varphi_t(\x))$, parameterized by $t\in[0,1]$, where $\varphi_0(\x)=(\x)$.
\end{dfn}
\begin{dfn}[Probability path]
Let $\PM$ denote the space of probability distributions on $\Man$. A probability path $\rho_t:[0,1]\to\PM$ interpolates between two distributions $\rho_0,\rho_1\in\PM$ indexed by $t\in[0,1]$.
$\rho_t$ is said to be \textbf{generated} by $\varphi_t$ if it \emph{pushes forward} $\rho_0$ to $\rho_1$ following $v_t$, \ie $\rho_t=[\varphi_t]_{\#}(\rho_0)$.
\end{dfn}
\begin{dfn}[RFM]
Given a probability path $\rho_t$, subject to the boundary conditions $\rho_0=\rho_{\mathrm{data}}$ and $\rho_1=\rho_{\mathrm{prior}}$, as well as an associated flow $\varphi_t$, learning a CNF by directly regressing $v_t$ through a parametric, neural network $g_{\beta}$, is called Riemannian flow matching. 
\end{dfn}
\begin{dfn}[Riemannian Conditional FM]
    Unfortunately, the vanilla RFM objective is intractable as we do not have access to the closed-form $u_t$ generating $\rho_t$. Instead, we can regress $g_{\beta}$ against a tractable \emph{conditional vector field} $v_t(\x_t\mid\z)$, generating a \emph{conditional probability path} $\rho_t(\x_t\mid\z)$ which can recover the target unconditional path by marginalization:
    \begin{equation}
        v_t(\x) = \int_{\Man} v_t(\x\mid\z)\frac{\rho_t(\x\mid\z)q(\z)}{\rho_t(\x)}\diff \mathrm{vol}_{\z}.
    \end{equation}
    Chen \& Lipman then define the following Riemannian conditional FM (RCFM) objective for learning as:
    \begin{equation}
        \Loss_{RCFM}(\beta) = \E_{t, q(\z), \rho_t(\x_t\mid\z)} d ( g_{\beta}(t,\x_t), v_t(\x_t,\z))^2,
    \end{equation}
    whose gradient is the same as that of RFM. Here, $t~\in\Unif(0,1)$ and $d(\cdot,\cdot)$ is the geodesic distance.
\end{dfn}

A simple variant of RCFM makes a particular choice of time scheduling, linearly decreasing the geodesic distance between $\x_t$ and $\x_1$ arriving at:
\begin{equation}\label{eq:RFMloss2}
\E_{t, q(\x_1), p(\x_0)} \left\lVert v_t(\x_t, t) + d(\x_0, \x_1) \frac{\grad~d(\x_t, \x_1)}{\norm{\grad~d(\x_t, \x_1)}^2_g}\right\rVert ^2_g
\end{equation}
This form will closely relate to our work as we clarify below. 

\paragraph{NRDF \& RFM}
In our work, we consider RFM on the product manifold of quaternions where $\Man\equiv\Quats^K, \x\equiv\pose$, and use the associated operators. We present in~\cref{alg:RFM} the Riemannian Flow Matching adapted to our problem,  articulated pose estimation. Note that, there are two fundamental differences: (i) the sampling of RFM and our sampling in~\cref{alg:sampledata_jan}, and (ii) the way the gradients are obtained. We now have a closer look into this.
\begin{algorithm}[t]
\caption{Training Riemannian Flow Matching for Learning on Articulated Bodies}
    \begin{algorithmic}[1]
    \Require Base distribution $p(\pose_0)$, target $q(\pose_1)$, initial parameters $\phi_0$ of a network $g_{\phi}$
    \Ensure Trained weights $\phi$
    \While {(not converged)}
        \State Sample $t\sim\Unif(0,1)$
        \State Sample training pose $\pose_1\sim q(\pose_1)$
        \State Sample noisy pose $\pose_0\sim p(\pose_0)$
        \State $\pose_t = \Exp_{\pose_0} ( t \cdot \Log_{\pose_0}(\pose_1))$ (\cf \cref{eq:RFMUpdate})
        \State Update $\phi_t$ by minimizing  $\loss_{\mathrm{RCFM}}$ in \cref{eq:RFMloss2}
    \EndWhile
    \end{algorithmic}
    \label{alg:RFM} 
\end{algorithm}

\paragraph{Sampling} As seen in~\cref{alg:RFM}, RFM samples time uniformly, $t \sim \mathcal{U}(0,1)$, and a target pose is obtained as:
\begin{equation}\label{eq:RFMUpdate}
\pose_t = \Exp_{\pose_0} ( t \cdot \Log_{\pose_0}(\pose_1)) \text{,}
\end{equation}
where $\pose_1 \sim q(\pose_1)$ is the data distribution and $\pose_0 \sim p(\pose_0)$ is the noise. Instead, our sampling algorithm presented in the main paper obtains a sample that is $h$ away from a training pose as: 
\begin{equation}
\hat{\pose} \leftarrow\text{Exp}_{\pose}(h \cdot \vb)  \text{,}
\end{equation}
where $\vb$ is directly sampled in the tangent space and normalized, and $h \sim \mathcal{P}$ for arbitrary $\mathcal{P}$ over $\mathbb{R}^+$. In contrast to RFM, we fix $d(\pose,\vb)=1$ through normalization and thus, $d(\hat{\pose}, \pose) = h$ (explicit control over distance).
Also note the time dependence. As our loss does not compute an explicit expectation over time, we can pre-compute all our training variables (nearest neighbors and the distances) offline. Having nearest neighbors as target distribution is unique to our work. 
As shown in~\cref{alg:NRDF}, this greatly simplifies the training loop leading to stable and fast training.
Moreover, note that, such pre-computation also allows for updating the nearest neighbor at each iteration during training. 

\paragraph{Obtaining gradients}
RFMs neural network minimizes~\cref{eq:RFMloss2} by explicitly predicting the steps.
When we would use our sampler from above, $d(\pose_0, \pose_1) = 1$ would hold, and the network would be trained to the distance, whose derivation (by backpropagation) provides $\frac{\grad~d(\pose_t, \pose_1)}{\norm{\grad~d(\pose_t, \pose_1}^2_g}$, which is the gradient of the distance field. Thus, in this case, the gradient of our network would coincide with the prediction of the flow matching network. Note that even with the flow matching sampling, the flow matching network prediction points in the same direction as the distance field gradient, it is just scaled by $d(\pose_0, \pose_1) \neq 1$.

\section{User Study for Evaluation Metrics}
\label{sec:user_study}
Previous studies~\cite{posemetric_GPD,posemetric_localfast} have highlighted a significant disparity between perceptual pose distance and commonly used metrics, such as differences in joint locations and orientations. The neural distance field model uses a certain distance metric to learn the relation between an arbitrary pose and the manifold of plausible poses by finding the nearest neighbor on the manifold. Consequently, \newmodel{} relies on the distance metric possessing specific properties: 1) the distance metric is well-defined, and continuous and 2) the distance metric is close to human perception. To assess these criteria, we conducted a user study comparing various metrics, including orientation, Euclidean-based distance metrics, and a combination of both. We now define each distance metric used in our user study.

\begin{algorithm}[t]
\caption{Neural Riemannian Distance Fields for Learning Articulated Pose Priors}
    \begin{algorithmic}[1]
    \Require Distribution $\mathcal{P}$, target distribution $q(\pose)$, initial parameters $\phi$ of a network $f_{\phi}$
    \Ensure Trained weights $\phi$
    \State Get training data samples $\mathcal{D}'$ via \cref{alg:sampledata_jan}
    \State Search nearest neighbour $\pose'$ and compute $d(\pose, \pose')$ for all $\pose \in \mathcal{D}'$
    \While {(not converged)}
        \State Sample a training pose $\pose_i$ from $\mathcal{D}'$
        \State Update $\phi$ by minimizing \\
        \qquad\qquad\qquad\qquad$\norm{ f_\phi(\pose_i) - \min_{\pose'\in\mathcal{D}} d(\pose_i, \pose')}$
    \EndWhile
    \end{algorithmic}
    \label{alg:NRDF} 
\end{algorithm}

\paragraph{Orientation-based metrics} We take the distance metric used by Pose-NDF~\cite{tiwari22posendf} as the candidate metric, denoting as $\Delta \q_p^l$.  We also adopt he quaternion geodesic distance $\Delta \q_g^l$, which has a more explicit physical interpretation and covers a wider range of values. For both metrics, we further calculate the distance between noisy poses and their nearest neighbor in global (relative to the root) frames, denoted as $\Delta \q_p^g$ and $\Delta \q_g^g$, respectively.

\paragraph{Euclidean-based metrics} Our Euclidean-based metrics involve calculating the average Euclidean distance over all body joints and a specific set of surface markers, denoted as \textbf{j2j} and \textbf{m2m} respectively. 

\begin{figure}[t]
    \centering
    \includegraphics[width=1.0\linewidth]{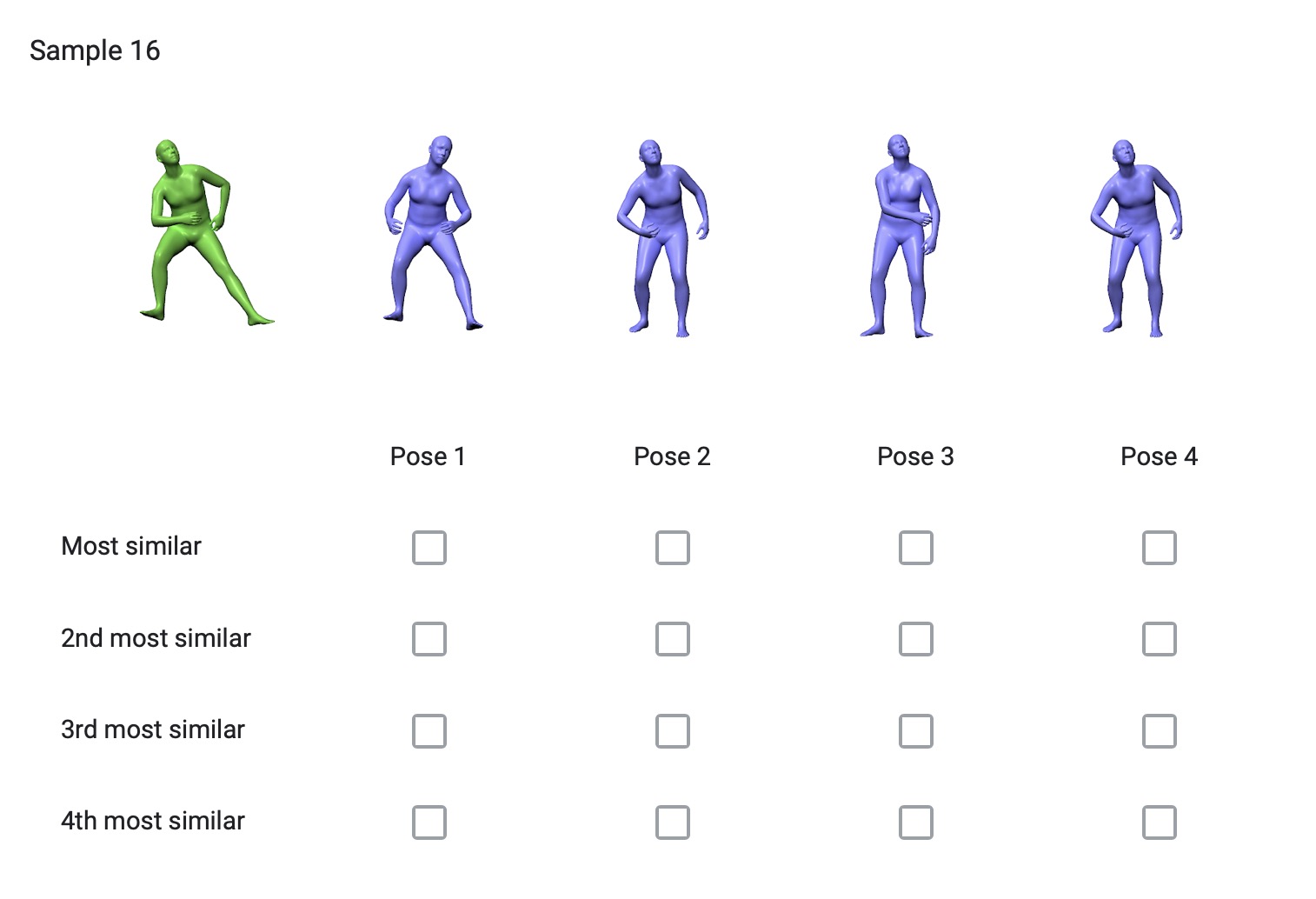}
    \caption{\textbf{User study for pose similarity assessment:} In our user interface, participants rank the similarity between a query pose (green) and its nearest neighbors (blue) from the AMASS dataset. These neighbors are obtained using different distance metrics.}
    \label{fig:user_interf}
\end{figure}

\paragraph{Combinations of orientation and Euclidean} In our observations, Euclidean-based metrics preserve the accurate overall shape of the body pose. However, they fall short of preserving the local twists of certain body joints. On the other hand, orientation-based metrics preserve precise local twists, yet they exhibit sensitivity to numerical values, resulting in divergent rotations even when the numerical values are close. To combine the strengths of both approaches, we introduce a hybrid metric, specifically defined as $\Delta \q + \text{m2m} = \text{m2m} + \lambda_{\q}\Delta \q_g^g $. This hybrid metric aims to leverage the advantages of Euclidean and orientation-based metrics, striking a balance that combines the faithful representation of the overall pose shape with the meticulous preservation of local joint twists. We set $\lambda_{\q} = 0.5$.

\paragraph{User study} We selected $\Delta \q_p^l$, $\Delta \q_g^g$, \textbf{m2m} and $\Delta \q + \text{m2m}$ as final candidates. We prepare 52 questions, each comprising the noisy pose and 4 NNs queried by a corresponding distance metric. Options are randomly shuffled in each question. As shown in~\figref{fig:user_interf}, users ranked the options from most similar to least similar, with the flexibility to assign the same rank to multiple options.

\paragraph{Result analysis} From a total of 54 responses, 32.79\% of users identified m2m as the most similar, while 30.09\% favored $\Delta \q + \text{m2m}$. For the second most similar, 29.72\% preferred $\Delta \q + \text{m2m}$, and 27.13\% chose $\Delta \q_g^g$. Following this, we use m2m and $\Delta \q + \text{m2m}$ as evaluation metrics for the ablation studies.

\section{Implementation Details}
\label{sec:more_implementation_details}
\begin{figure*}[h]
    \centering
    \includegraphics[width=1.05\textwidth]{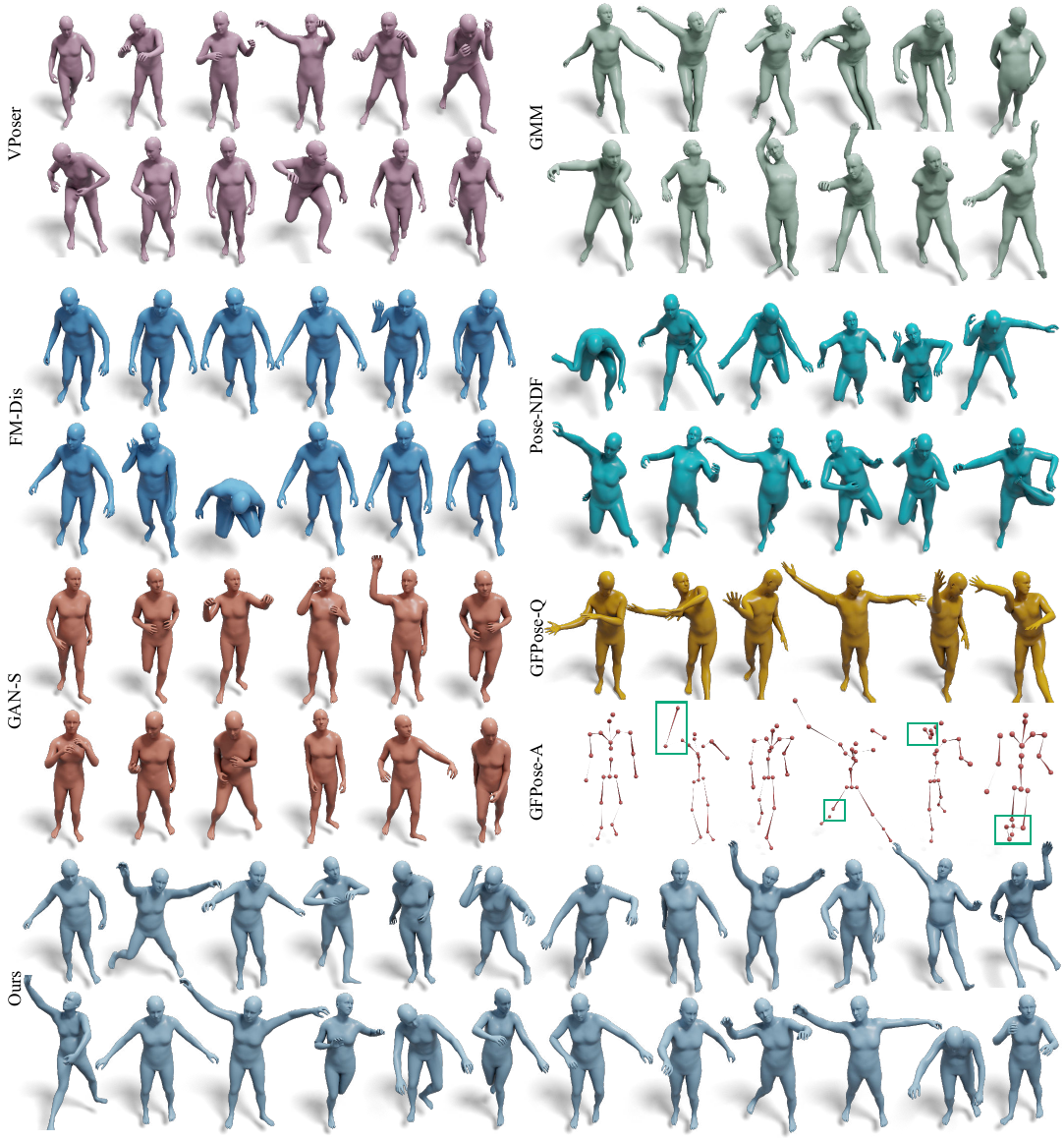}
    \caption{\textbf{Pose generation}: We compare pose generation results of our method with VPoser~\cite{SMPL-X:2019}, GMM, FM-Dis, Pose-NDF~\cite{tiwari22posendf}, GAN-S~\cite{Davydov_2022_CVPR}, GFPose-A~\cite{ci2022gfpose} and GFPose-Q. In comparison to VPoser, our method produces more diverse results. Furthermore, when compared to GMM, FM-Dis, and Pose-NDF, our method generates more realistic poses.}
    \label{fig:pose_gen_supp}
    \vspace{5pt}
\end{figure*}

In this section, we introduce the experimental setup for data preparation, network training, baselines, and optimization-based downstream tasks such as partial-IK solver and image-based pose estimation.

\begin{table*}[th]
	\begin{center}
		\centering
		
 		\resizebox{1.0\textwidth}{!}{
		\begin{tabular}{l|ccc|ccc|ccc}
		    \toprule
		     \multirow{2}{*}{Method}  & \multicolumn{3}{c|}{Occ. Single Arm}& \multicolumn{3}{c|}{Only End Effectors Visible} & \multicolumn{3}{c}{Occ. Legs}\\
			&  \fid{} $\downarrow$ & APD (in cm) $\uparrow$\  & $\distnn$ (in rad) $\downarrow$\ & \fid{} $\downarrow$ & APD (in cm) $\uparrow$\   & $\distnn$ (in rad) $\downarrow$\  & \fid{} $\downarrow$ & APD (in cm) $\uparrow$\  & $\distnn$ (in rad) $\downarrow$\  \\
			\hline
			 VPoser-Random      & $1.145^{\pm.265}$  & $3.085^{\pm.642}$ & $0.066^{\pm.001}$ &  $0.681^{\pm.091}$  &  $8.330^{\pm.783}$  & $0.067^{\pm.000}$ & $0.748^{\pm.141}$  & $6.650^{\pm.951}$   & $0.058^{\pm.001}$ \\
			 Pose-NDF~\cite{tiwari22posendf}  & $1.460^{\pm.233}$  & $16.445^{\pm2.415}$ & $0.622^{\pm.001}$   &  $2.081^{\pm..114}$ &  $31.524^{\pm.872}$  & $0.738^{\pm.001}$ & $3.015^{\pm.162}$  &  $24.831^{\pm1.328}$  &  $0.677^{\pm.001}$ \\
			 FM-Dis  & $1.150^{\pm.253}$  &  $4.967^{\pm.798}$ & $0.210^{\pm.001}$    &  $1.010^{\pm.103}$ & $10.812^{\pm1.451}$  & $0.347^{\pm.002}$ & $0.976^{\pm.158}$ &  $6.886^{\pm1.664}$  &  $0.226^{\pm.020}$  \\
			  \textbf{Ours}  & $1.306^{\pm.259}$ &  $5.892^{\pm.236}$  &   $0.132^{\pm.000}$ & $0.964^{\pm.117}$ & $10.388^{\pm.820}$ & $0.137^{\pm.001}$ & $0.899^{\pm.170}$  & $6.705^{\pm.613}$   & $0.125^{\pm.001}$ \\
			\bottomrule
            \end{tabular}
            }\vspace{-2mm}
        \caption{\textbf{Quantitative results for IK Solver from partial/sparse joints.} We run all evaluations 20 times, $\pm$ indicates the 95$\%$ confidence interval. We evaluate under 3 settings: \textbf{Occ. Single Arm},   \textbf{Only End Effectors}  (wrists and ankles) \textbf{Visible}, and \textbf{Occ. Legs}  \vspace{-7mm}}
		\label{tab:ikjoints}
	\end{center}
\end{table*}

\subsection{Data Preparation}

\paragraph{Training data} For training, we use a subset of the AMASS~\cite{AMASS:2019}. We follow the training split of AMASS used in VPoser~\cite{SMPL-X:2019} and Pose-NDF~\cite{tiwari22posendf} and assume that the AMASS training set can sufficiently explain a comprehensive and valid human pose manifold. 

To pre-process the AMASS dataset, we crop the central 80\% of each motion sequence to focus on the most informative part of the data. We apply subsampling at a rate of 30 Hz, resulting in a total of 4 million clean poses. This is similar to VPosers and Pose-NDFs training setup.
To create negative training samples along with their corresponding ground truth distances to the manifold, we sample noisy poses using  \cref{alg:sampledata_jan}, with $\mathcal{P} = \mathcal{N}(0, \pi / 4)$. Following Pose-NDF~\cite{tiwari22posendf}, to speed up the NN search process, we adopt a multi-step mechanism for querying the NN of each noisy pose.
For the first stage, we implement $k^{\prime}\text{NN}$ using FAISS~\cite{johnson2019billion}, get $k^{\prime}$ candidates. For the second stage, we use the quaternion geodesic distance to find exact $k$ neighbors from these $k^{\prime}$ candidates. 
In our implementation, we set $k^{\prime} = 1000$ and $k=1$. Note that we determine the final distance by considering only the closest NN, deviating from the approach in Pose-NDF~\cite{tiwari22posendf}, where the average distance over the top 5 NNs is computed. This is motivated by the observation that the top 5 NNs may exhibit discontinuities, and averaging their distances tends to over-smooth the manifold boundary. 

\paragraph{Evaluation and validation} For validation, we utilize the validation split of the AMASS dataset, specifically we use Human Eva, MPI-HDM05, SFU, and MPI Mosh. For testing the accuracy and convergence speed across various baselines, we take the test split of AMASS dataset, specifically, we use SSM-Synced and Transitions. The distance values are computed in reference to the training data.

\subsection{Network Training}

~\cref{alg:NRDF} shows our training procedure. Specifically, we employ a hierarchical neural implicit network to implement \newmodel{}, following the approach outlined in~\cite{tiwari22posendf}. The network takes a quaternion representation of SMPL pose as input and produces a scalar distance field as output. We adopt a two-stage training strategy. For the first stage, each training batch comprises a balanced combination of 50\% noisy poses and 50\% clean poses. Subsequently, to enhance generalization to downstream tasks, in the second stage, we fine-tune our model using poses sampled from a standard normal distribution $\mathcal{N}(0, \mathbf{I}) \in \mathbb{R}^{4K}$. We set the learning rate to 1e-4. The entire training process requires 8-10 hours with a single GeForce RTX 3090 GPU. 

\subsection{Baseline Details}

In this section, we present implementation details of baselines. Our evaluation focuses on the \textbf{pose denoising} task, where we compare finally converged poses with their ground truth nearest neighbors. We sample 20k noisy poses by using  \cref{alg:sampledata_jan} based on AMASS~\cite{AMASS:2019} test set. We now first investigate the significance of our Riemannian projection (\textbf{Ours} v/s \textbf{Ours w/o \RDFGrad{}}) and sampling method (\textbf{Pose-NDF} v/s \textbf{Ours w/o \RDFGrad{}}). This is followed by comparison with a closely related Riemannian Flow Matching~\cite{chen2023riemannian} based distance field (\textbf{{Ours v/s FM-Dis}}) and an ablation on distance v/s gradient field modeling.

\paragraph{Ours v/s Ours w/o \RDFGrad{}} In this study, we evaluate the significance of our novel adaptive-step Riemannian gradient descent algorithm, termed as \RDFGrad{}. For \textbf{Ours w/o \RDFGrad{}} we use standard stochastic gradient descent (SGD) in Euclidean space. The results presented in~\tabref{tab:ablation} illustrate that our approach achieves 6$\times$ acceleration in convergence speed thanks to the gradient update on the Riemannian quaternion manifold. Please refer to the supplementary video for qualitative comparisons. The convergence criterion is based on the absolute difference between the previously predicted distance and the current predicted distance being less than 1e-5.

\paragraph{Ours w/o (\RDFGrad{}, $\newdist{}$) v/s Pose-NDF} In this study we evaluate the significance of our novel training data sampling strategy. For this we compare \textbf{Pose-NDF}, which uses a naive sampling strategy with \textbf{Ours w/o (\RDFGrad{}, $\newdist{}$)}, which is basically Pose-NDF with our novel sampling strategy. From~\tabref{tab:ablation}, we observe that just changing the sampling strategy, reduces the m2m distance from 25.04 cm to 15.16 cm, which shows the significance of the distance distribution of training data.

\paragraph{Ours v/s FM-Dis} In order to connect with the recent Riemannian flow matching, we introduce a new baseline, called \textbf{FM-Dis}, which extends RFM to model the pose prior as a distance field. Instead of predicting the gradient, our variation \textbf{FM-Dis} predicts the distance between $\pose_t$ and the corresponding clean pose $\pose_1$ without recomputing the nearest neighbor. Specifically, we minimize $\Loss_{\mathrm{FM-Dis}}$, given by~\eqnref{eq:rfm-dis}, where $\pose_t$ is sampled by using \cref{eq:RFMUpdate}, 
\begin{align}
&\Loss_{\mathrm{FM-Dis}} (\phi) \\
&= \E_{t\sim {\cal U}(0,1), q(\pose_1), p(\pose_0) } \| v_{\phi}(\pose_t) - d(\pose_t, \pose_1) \|_g^2  \label{eq:rfm-dis} 
\end{align}

It is apparent that $\pose_t$ is evenly interpolated between the noise and a particular clean sample, which stands in contrast to our distribution, where the number of samples gradually decreases as one moves outward from the manifold. As shown in~\figref{fig:pose_gen_supp}, FM-Dis tends to generate poses close to the mean. 
We show qualitative comparisons in our supplementary video.

\paragraph{Distance v/s gradient prediction} To connect our model with diffusion / score-based methods and flow matching-based methods, we implement \textbf{Gradient prediction w/o time}, \textbf{FM-Grad w/ time} and \textbf{GFPose-Q}. These approaches explicitly predict either the gradient or the gradient direction, while our method derives the gradient of distance with respect to the input pose through network back-propagation. 
Predicting full gradients (including length) without time-step conditioning is challenging for neural networks, leading to significant errors. Therefore, \textbf{Gradient prediction w/o time} is designed to predict the gradient direction (with normalized length) between the noisy pose and its nearest neighbor only. Noisy poses are sampled using \cref{alg:sampledata_jan} in this case. We additionally incorporate Riemannian flow matching (RFM)~\cite{chen2023riemannian} into our experiment, denoting as \textbf{FM-Grad}. Different from RFM, we maximize the cosine similarity between the network prediction and gradient, thus, minimizing $\Loss_{\mathrm{FM-Grad}}$, given by ~\eqnref{eq:rfm-grad}, where $t \sim {\cal U}(0,1)$ as above and $\pose_t$ is obtained in the same manner as FM-Dis. We set $T = 1000$ during training. 

\begin{align}
\label{eq:rfm-grad}
&\Loss_{\mathrm{FM-Grad}}(\phi) \\
&= - \E_{t, q(\pose_1), p(\pose_0) } \left| \frac{v_{\phi}(\pose_t, t) \cdot \text{Log}_{\pose_t}(\pose_1)} { \| v_{\phi}(\pose_t, t)  \|_g^2  \| \text{Log}_{\pose_t}(\pose_1) \|_g^2}   \right| \textrm{.}\nonumber
\end{align}

For test-time projection, we follow ~\eqnref{eq:RGD} using $v_{\phi}(\pose_t, t)$ as the gradient. We set $\tau = 0.01$ and the initial time-step $T^{\prime} = 200$. For training Gradient prediction w/o time and FM-Grad, we employ the same network architecture as GFPose~\cite{ci2022gfpose}. The convergence criterion is based on the absolute difference between the predicted gradient norm at $t+1$ and $t$ being less than 1e-5. Regarding \textbf{GFPose-Q}, we retrain it using the quaternion representation on the AMASS dataset, with $\emptyset$ conditioning.  Since gradient prediction is less accurate than distance prediction, results based on gradient prediction tend to exhibit either over-correction or unrealistic poses.
Please refer to the supplementary video for qualitative comparisons.

\begin{figure*}[t]
    \centering
    \includegraphics[width=1.05\textwidth]{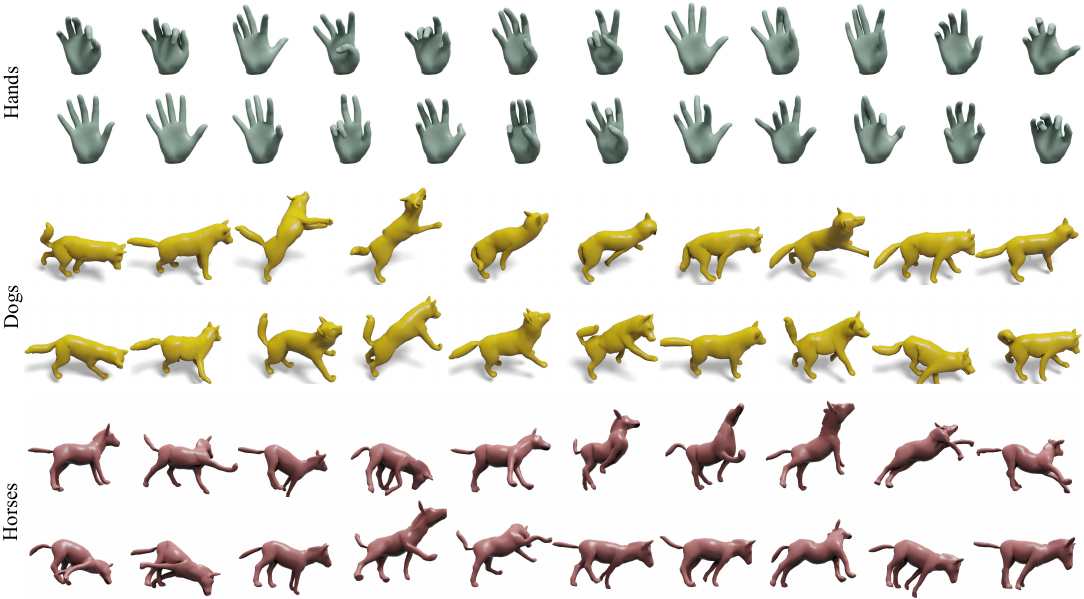}
    \caption{\textbf{Animal and Hand pose generation:} We generate diverse animal and hand poses using \newmodel{}.}
    \label{fig:ani_gen}
\end{figure*}

\subsection{IK Solver Setup}
We utilize the AMASS test set and compute ground truth marker or joint positions through forward kinematics. The overall loss function, based on  \cref{eq:opt_generic}, encompasses a data term defined by the L2 loss between predicted marker/joint locations and observations.
Given that most off-the-shelf optimizers in PyTorch are SGD-based or its variations, and there is no optimizer designed for quaternions in geoopt~\cite{geoopt2020kochurov}, we introduce a custom optimizer specifically designed for \RDFGrad{}. This involves obtaining the Euclidean gradient returned by the network and projecting it onto the Riemannian quaternion manifold using Eq. \ref{eq:egrad2rgrad}. We plan to release our code for public use and stimulating future research.
During optimization, for VPoser, VPoser-Random and Pose-NDF, we set  $\lambda_{\pose} = 0.1$, $\lambda_{\beta} = 0.05$ and $\lambda_{\alpha} = 0.001$. For FM-Dis and \newmodel{}, we set $\lambda_{\pose} = 5.0$, $\lambda_{\beta} = 0.05$ and $\lambda_{\alpha} = 0$. Concerning our \RDFGrad{}-based optimizer, for an effective initialization of the prior loss, we exclude the data term in the first stage and optimize only the prior term. Subsequently, we combine all loss terms for joint optimization.
The stopping criterion for all experiments is set as $\text{MPJPE} = 3 \text{ cm}$.

\begin{figure}[h]
    \centering
      \begin{overpic}[width=0.5\textwidth,unit=1bp,tics=20]{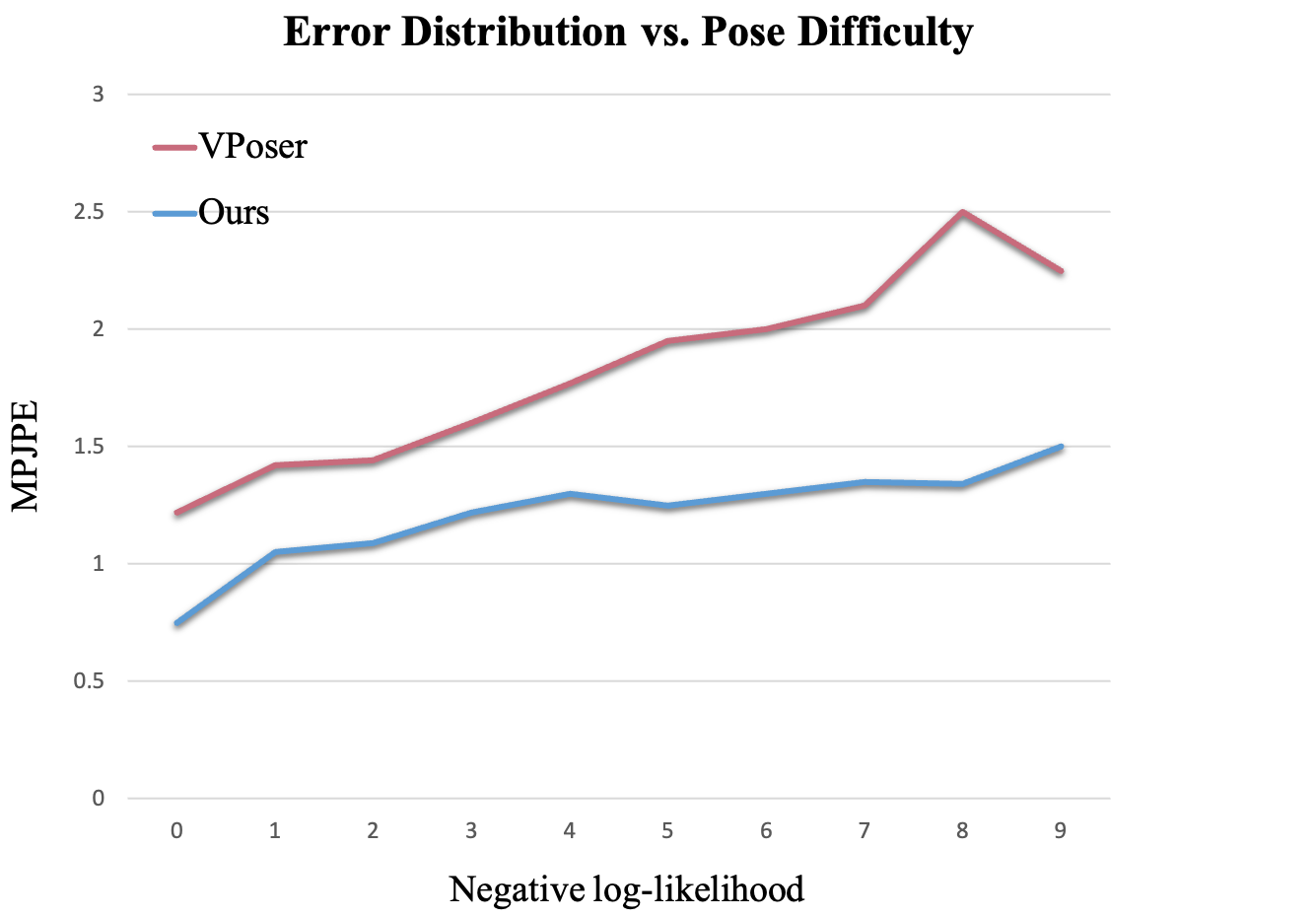}
      \put(1,72){\rotatebox{90}{\colorbox{white}{\scriptsize MPJPE (cm)}}}
        \put(37,143){\colorbox{white}{\scriptsize VPoser}}
        \put(37,130){\colorbox{white}{\scriptsize Ours}}
        \put(80,5){\colorbox{white}{\scriptsize Negative log-likelihood}}
    \end{overpic}
    \caption{\textbf{Error Distribution vs. Pose Difficulty:} X axis represents the relative negative log likelihood (NLL) while Y axis represents the MPJPE between the result joint locations and corresponding observations.}
    \label{fig:pose_difficulty}
\end{figure}

\begin{figure*}[t]
    \centering
    \includegraphics[width=1.05\textwidth]{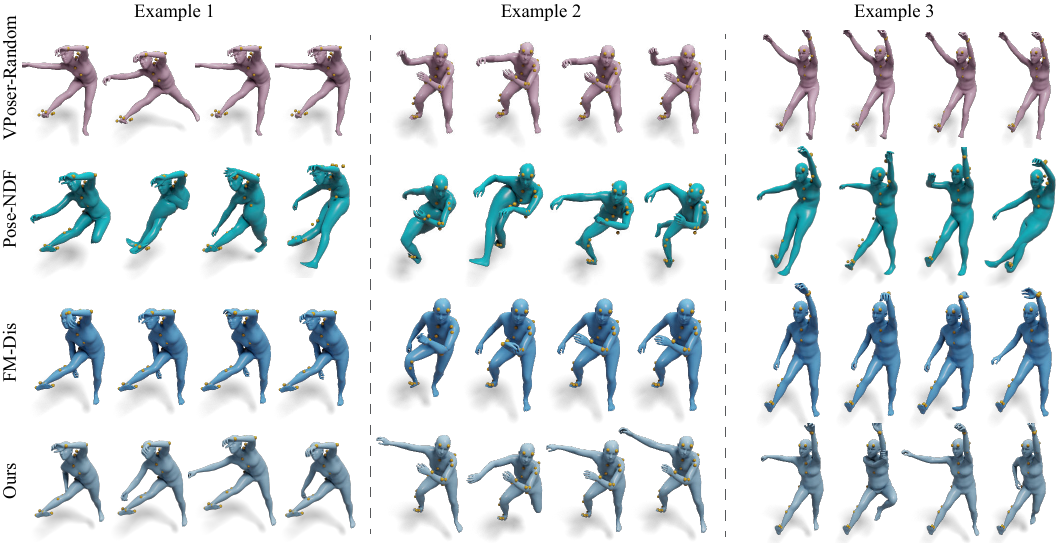}
    \caption{\textbf{IK Solver for one arm and one leg occluded:} Given partial observation, where one leg and one arm are occluded, we perform 3D pose completion. We observe that VPoser~\cite{SMPL-X:2019} based optimization generates realistic, yet very limited diversity in poses. Pose-NDF~\cite{tiwari22posendf} generates more diverse, but unrealistic poses. FM-Dis also generates limited diversity in poses. NRDF generates diverse and realistic poses as compared to the aforementioned pose priors.}
    \label{fig:ik_armleg_supp}
\end{figure*}

\subsection{Image-based Tasks Setup}
For evaluating the impact of our prior on human pose estimation from images, we use 3DPW dataset~\cite{vonMarcard2018}. We conduct the evaluation on the test split of 3DPW, where ground truth for a single person in the image is available. We crop the images using the GT 2D keypoints and discard the images where SMPLer-X doesn't predict any person. This amounts to roughly ~5k images. We use ``SMPLer-X-S32'' for predicting SMPL-X parameters and use optimization-based processing to convert SMPL-X parameters to SMPL~\cite{SMPL-X:2019}. In particular, we want to refine network prediction using optimization-based refinement. We use SMPLer-X~\cite{cai2023smpler} for predicting human pose $\hat{\pose}$ and shape $\hat{\shape}$ from images and then use optimization loss mentioned in \cref{eq:opt_generic} and \cref{eq:image} to refine the predictions. For VPoser, we optimize the latent vector $z$ of VAE, where $z$ is initialized from the VAE encoding of predicted pose or $\hat{\pose}$. More specifically  $z_{\mathrm{init}} = f_{\mathrm{VE}}(\hat{\pose})$, where $f_{\mathrm{VE}}$ is encoder of VPoser. For Pose-NDF, FM-Dis, and Our prior optimization, we simply optimize for $\pose$ and the variable is initialized using $\hat{\pose}$. We also optimize for SMPL shape ($\shape$) parameters in both setups. For the evaluation metric, we have used PA-MPJPE (Procrustes aligned-MPJPE), PA-PVE (Procrustes aligned-per-vertex error).

\section{Additional Results}
\label{sec:more_results}
We now provide additional qualitative and quantitative results.

\begin{figure*}[h]
    \centering
    \includegraphics[width=1.05\textwidth]{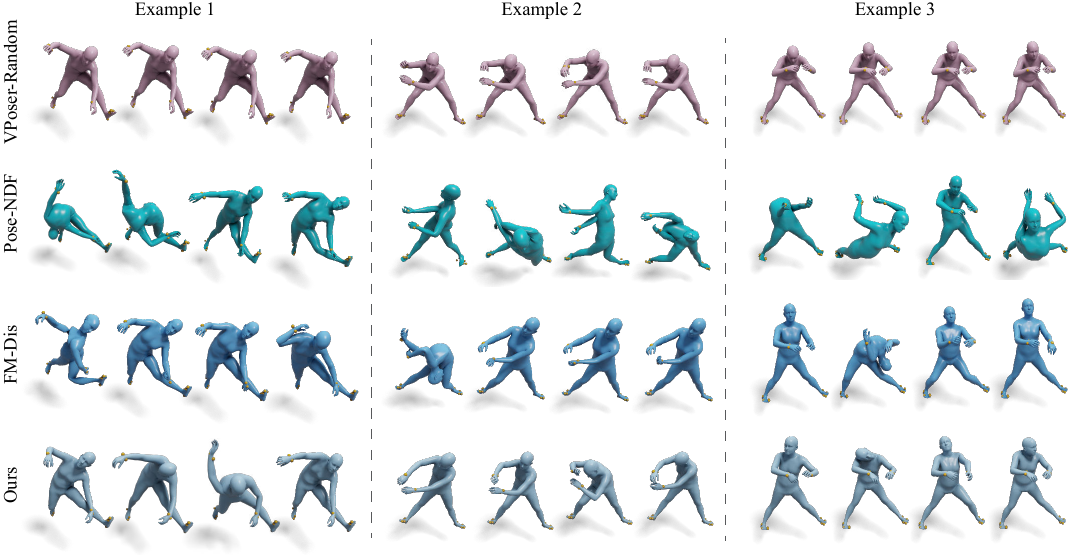}
    \caption{\textbf{IK Solver for visible end-effectors:} Given partial observation, where only end-effectors are visible (yellow markers), we perform 3D pose completion. We observe that VPoser~\cite{SMPL-X:2019} based optimization generates realistic, yet very limited diversity in poses or almost similar poses. Pose-NDF~\cite{tiwari22posendf} and FM-Dis result unrealistic poses for such sparse observations. NRDF generates diverse and realistic poses as compared to the aforementioned pose priors.}
    \label{fig:ik_end_supp}
\end{figure*}

\begin{figure*}[h]
    \centering
    \includegraphics[width=1.05\textwidth]{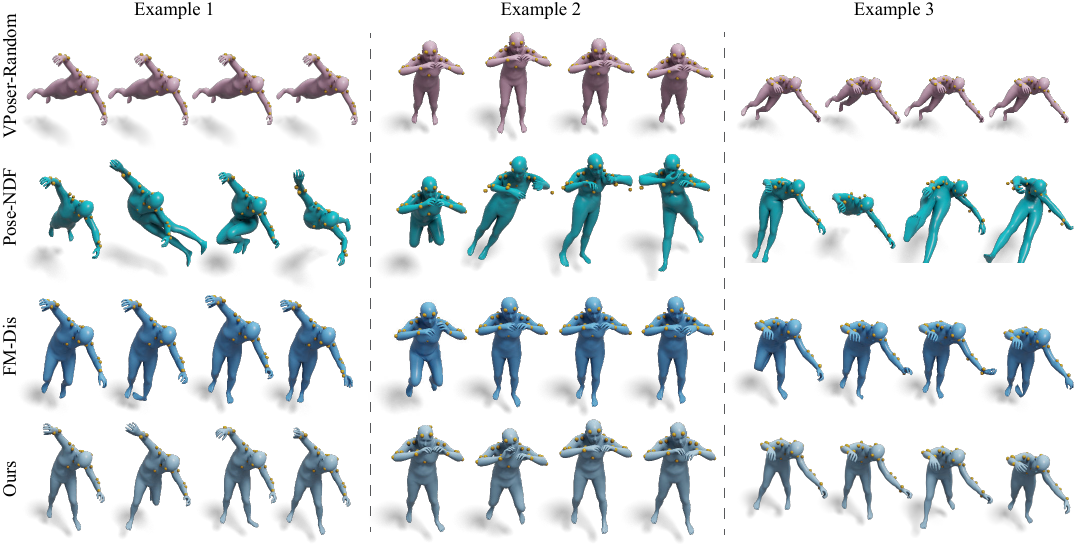}
    \caption{\textbf{IK Solver for occluded legs:} Given partial observation, where only one leg is occluded, we perform 3D pose completion. We observe that VPoser~\cite{SMPL-X:2019} based optimization generates realistic, yet very limited diversity in poses or almost similar poses. Pose-NDF~\cite{tiwari22posendf} results unrealistic poses. FM-Dis generates almost similar poses and has no diversity. NRDF generates diverse and realistic poses as compared to the aforementioned pose priors.}
    \label{fig:ik_upper_supp}
\end{figure*}

\begin{figure*}[t]
    \centering
    \includegraphics[width=1.05\textwidth]{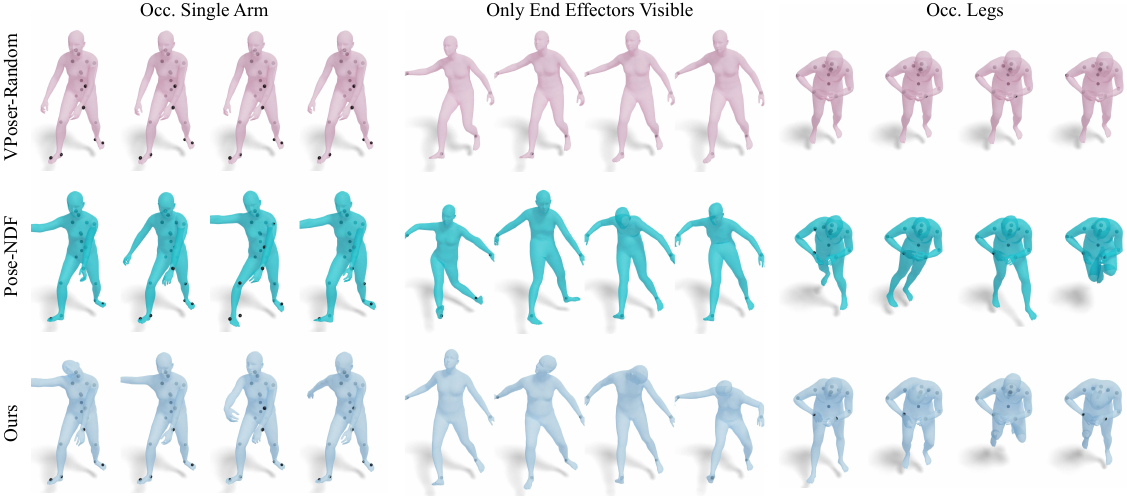}
    \vspace{-2mm}
	\caption{\textbf{IK Solver from partial/sparse joints:} Given partial observation (black joints), we perform 3D pose completion. We observe that VPoser~\cite{SMPL-X:2019} based optimization generates realistic, yet less diverse poses. Pose-NDF~\cite{tiwari22posendf} generates more diverse, but sometimes unrealistic poses, especially in case of very sparse observations. \newmodel{} generates diverse and realistic poses in all setups.}
	\label{fig:partialIK_joints}
\end{figure*}

\subsection{IK solver from partial surface markers or body joints}

We show more IK results from partial observations. For surface markers, we observe that for Occ. Single Arm and Visible end-effectors setup, our model generates more diverse poses based on APD. We also show qualitative results in~\figref{fig:ik_end_supp}. Note that despite the numerical diversity of VPoser, it exhibits fewer diverse poses for occluded legs compared to \newmodel{}. As depicted in~\figref{fig:ik_upper_supp}, the legs of VPoser tend to move together without interaction between them, which could also result in a large APD value. In contrast, our method demonstrates more diverse leg poses, including bending of knees. We provide results for another setup in~\figref{fig:ik_armleg_supp}, where one arm and one leg are occluded.
Given that body joints are more underconstrained and challenging, we further evaluate our IK solver on partial body joint observations.~\tabref{tab:ikjoints} and~\figref{fig:partialIK_joints} illustrate the IK results, showcasing that our method achieves accurate IK while maintaining more diversity.

\subsection{Monocular 3D Pose Estimation from Images}

\begin{figure*}[t]
    \centering
    \includegraphics[width=1.05\textwidth]{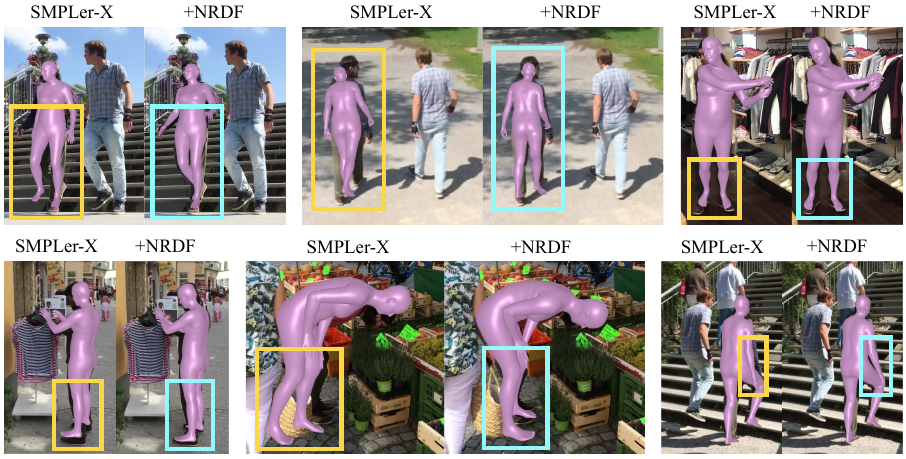}
	\caption{ \textbf{3D pose and shape estimation from images:} (Top): We refine the results of SMPLer-X~\cite{cai2023smpler} network prediction using \newmodel{} based optimization pipeline. }
	\label{fig:hps_suppl}
\end{figure*}

We provide more qualitative results for 3D pose estimation from images in ~\figref{fig:hps_suppl}.

\subsection{More Pose Generation Results}

~\figref{fig:pose_gen_supp} shows additional generation results from different methods. Note that poses generated by GMM can appear unrealistic and can have implausible bends around joints such as elbow or shoulder joints, as shown in~\figref{fig:pose_gen_supp} (top-right). VPoser~\cite{SMPL-X:2019} tends to generate more standing and less diverse poses. This is attributed to Gaussian assumption of the latent space. FM-Dis also generates less diverse poses, close to mean poses and some times results in unrealistic poses as well. Pose-NDF~\cite{tiwari22posendf} generates diverse poses, in terms of bends around knees, elbows \etc but at the same time, it results in implausible poses. This is attributed to the inaccurate distance field of Pose-NDF. GAN-S~\cite{Davydov_2022_CVPR} also tends to generate less diverse pose, as compared to Pose-NDF. We also compare with a diffusion-based model GFPose~\cite{ci2022gfpose}. We retrain GFPose on AMASS dataset and call it GFPose-A. Since this is joint-location based model, we observe that the generated results might have inconsistent bone lengths, as highlighted in~\figref{fig:pose_gen_supp}. We also retrain GFPose on quaternions, denoting as GFPose-Q, which similarly generates less diverse and unrealistic poses.

We further show more results for hand and animal pose generation in~\figref{fig:ani_gen}.

\subsection{Error Distribution vs. Pose Difficulty}

Poses generated by VPoser~\cite{SMPL-X:2019} exhibit a tendency to cluster around mean poses, given it is based on Gaussian assumptions. In this subsection, we explore the correlation between location error and pose difficulty in partial Inverse Kinematics (IK) tasks. The observed relationship is visually depicted in~\figref{fig:pose_difficulty}. It's noteworthy that as the ground truth pose becomes less common (indicated by larger Negative Log-Likelihood (NLL) values), the difference between Mean Per Joint Position Error (MPJPE) of VPoser and \newmodel{} tends to increase. However, \newmodel{} consistently maintains a smaller error, remaining under 1.5 cm. The first column of~\figref{fig:ik_armleg_supp} (foot part of VPoser) also shows that VPoser fails to meet the observations when the given pose is uncommon.

\end{document}